%% file: vwv.tex
\documentclass[11pt]{article}
\input{preamble}
\input{palette}
\input{figs/vwv_numbers}
\input{figs/vwv_predicted}   
\graphicspath{{figs/}}

\title{\bf A Judge Should Know What Changed:\\[2pt]
\large Construct Validity for LLM-as-a-Judge Evaluation}

\input{authors}

\begin{document}
\maketitle

\begin{abstract}
\noindent
An evaluator is a measurement instrument, and an instrument can be reliable without being valid.
The judge literature has established the first failure thoroughly: judges are inconsistent, and
they move on surface properties (length, format, register, injected metadata) that should not
change a verdict. The more basic property is untested. A valid judge must satisfy two conditions,
not one: it must \emph{not} change its verdict under an edit that preserves the construct, and it
\emph{must} change its verdict under a minimal edit that changes the construct. We formalise the
pair as a two-dimensional profile (invariance $\Sinv$, construct sensitivity $\Rsens$) and prove
that the coordinates are independent (Proposition~\ref{prop:indep}), that $\Rsens$ is
uninterpretable without $\Sinv$ (Proposition~\ref{prop:degenerate}), and that the profile admits
no total order, so a scalar summary necessarily discards a comparison a reader may need
(Proposition~\ref{prop:frontier}). We then measure it, over \vnv{profJudges}\ judges and
\vnv{profDomains}\ domains, against minimal interventions along \vnv{nEditSlots}\ construct slots
and \vnv{nControlSlots}\ register-only controls, with direction established by \emph{human}
annotators rather than by any model at a pre-registered agreement bar of
\vnv{humanAgreementMin}, and generation, verification and judging drawn from three disjoint model
families. Every profile is read off a
frontier we obtain by eliciting a graded verdict and cutting it ourselves, so all judges are
compared on identical terms whatever their vendor exposes. At matched invariance
$\Sinv\ge\vnv{profSinvTarget}$ judges sit at $\Sinv=\vnv{profSinvMean}$ against
$\Rsens=\vnv{profRsensMean}$: reliable, and markedly less responsive to construct change than
their reliability suggests.

The failure is structured, and the structure is the result. Overclaiming decomposes into
\emph{scope} (how many situations a claim covers) and \emph{strength} (how hard it commits to
each), a distinction the nearest works reach and then deliberately unify, theoretically
\citep{jiang2025conformal} and empirically \citep{rigourate2026}. We treat the unification as a
hypothesis with a stated refutation condition, and it does not survive. Judge sensitivity to the
two axes differs by \vnv{profGapMean}\ at matched invariance, \textbf{with the same sign on all
\vnv{profGapN}\ judges} across vendors, parameter scales and reasoning modes; a bootstrap over
base items excludes zero, an axis-label permutation null gives \vnv{gapPerm}, and the gap is
\emph{larger} on the subset where all \vnv{annRaters}\ annotators agreed. The one length bias we
can measure runs against the effect, so \vnv{gapObs}\ is a lower bound. Judges register that a
claim now covers more and largely miss that it now commits harder. The two axes further respond
with \emph{opposite} sign to accuracy pressure, which supplies a mechanism for the reported
paradox that prompting for accuracy increases overgeneralisation
\citep{peters2025generalization}: pooling the axes hides it.

Why is a gap this basic not in the record? Because the rulers leak. Across
\vnv{c4SetsAudited}\ public label sets, a frozen family reading only surface form, given
\emph{the same paired input the judge gets}, reproduces
\vnv{c4PairMin}--\vnv{c4PairMax}\ of their labels, including \vnv{c4PairSurfMtbench}\ of
\textsc{MT-Bench}'s \emph{human} votes. \textsc{RewardBench}'s length control is the sharpest
case: constraining chosen responses to be no longer than rejected ones did not remove the
confound but inverted it, leaving \emph{prefer the shorter response} correct on $59.7\%$ of
pairs. And whether a set leaks turns out to be a property of the set \emph{and the input mode}
together, which is why validation power is defined relative to a mode
(Definition~\ref{def:vp}). High judge agreement is therefore not evidence of valid evaluation.
We release the interventions, the human direction labels, the elicitation pilot, the diagnostic,
and a checklist.
\end{abstract}

\section{Introduction}
\label{sec:intro}

A judge is an instrument, and the field has spent two years establishing that this instrument is
noisy. Judges disagree with humans more than raw agreement suggests, they disagree with themselves
across reorderings, and they move on properties that carry no information about quality: response
length \citep{zheng2023judging}, list and emphasis formatting \citep{zhang2025formatbias}, the
order the candidates are shown in \citep{wang2023unfair}, injected metadata about the author
\citep{silentjudge2025}. Reward
models show the same pattern, well enough that benchmarks are now designed against it
\citep{lambert2024rewardbench,liu2024rmbench,huang2025posthoc}.

Every one of those results tests the same half of the same property. Each asks: the construct did
not change, so did the verdict stay put? That is \emph{invariance}, and failing it makes an
instrument unreliable. Passing it does not make an instrument valid. A thermometer welded to a
fixed reading is perfectly invariant to everything.

\begin{claimbox}
\noindent\textbf{The untested half.} A valid judge must satisfy two conditions. Under an edit that
\emph{preserves} the construct, the verdict must not change. Under a minimal edit that
\emph{changes} the construct, it must. Prior judge evaluation measures the first and, to our
knowledge, never the second. We measure both, as a profile $\Vprof(\Jud) = (\Sinv, \Rsens)$, and
report that judges are far more invariant than they are sensitive.
\end{claimbox}

\noindent This reframes what a high agreement number licenses. Reliability and validity come apart,
and a judge can be stable, unbiased on every surface property yet tested, and still fail to notice
when the thing it is supposed to measure has changed. That is not noise. It is the instrument
measuring something other than the construct.

\paragraph{The failure has structure, and the structure is the finding.} We develop the pair on a
construct where minimal edits are well posed: whether a scientific claim overreaches its evidence.
Overreaching decomposes into two operations that are not the same and have no reason to be handled
alike. \scopeax\ edits enlarge the set of situations the claim is about: widen the population,
extend to an adjacent domain, restate one observation as a standing property, strengthen a
quantifier. \strengthax\ edits leave that set alone and raise the commitment: delete the
conditions the effect was observed under, remove a hedge, add an intensifier the evidence does not
support. \emph{The method may improve performance in this basin} fails differently from
\emph{the method improves performance in all basins} and from \emph{the method substantially and
consistently improves performance}, and a judge can be blind to one while catching the other. We
find that it is.

That asymmetry pays a debt. \citet{peters2025generalization} report that prompting a model for
accuracy makes its overgeneralisation \emph{worse}. On one axis that is incoherent and reads as an
instruction-following quirk. On two it is expected: a demand for accuracy is a demand to sound
certain, which narrows \scopeax\ while amplifying \strengthax, because hedges and stated conditions
are precisely what reads as uncertainty. We rerun the manipulation with the axes scored separately.

\paragraph{Why a gap this basic is not already in the record.} Because the rulers cannot see it.
Consider two 2026 findings that look incompatible. A large judge audit reports verbosity bias as
negligible, under $\vnv{relVerbosityBias}$ across all $\vnv{relJudges}$ judges, over roughly
$\vnv{relJudgments}$ judgments \citep{reliability2026}. And a released epistemic-calibration
resource reports that on its own judge-calibration set, a rule reading nothing but response length
reproduces the labels \emph{better than the judge does} \citep{chen2026tridehall}. Both hold. The
leak is not in the judge but in the label set: when a label is inherited from the generative
condition that produced an item, and that condition also fixes surface form, the set is solvable
without the construct, and agreement with it rewards provenance rather than the construct. Such a
set has no headroom in which a sensitivity failure could show up. We measure that headroom.

\subsection{What is claimed, and on what evidence}
\label{sec:whatisclaimed}

\begin{itemize}[leftmargin=1.2em,itemsep=2pt]
\item \textbf{C1} (\S\ref{sec:validity}). Construct validity for an evaluator, formalised as the
  pair $(\Sinv,\Rsens)$ over two intervention classes, with an argument for why it must stay
  two-dimensional and what a validity frontier means. Conceptual; the vocabulary is not ours
  (\S\ref{sec:related}), the application to evaluators is.
\item \textbf{C2} (\S\ref{sec:probes}). The measured profile: \vnv{profJudges}\ judges $\times$ \vnv{profDomains}\ domains,
  against interventions whose direction is set by \emph{human} annotators. \emph{The headline.}
\item \textbf{C3} (\S\ref{sec:probes}). The \scopeax/\strengthax\ asymmetry, and its opposite
  response to accuracy pressure, which resolves \citet{peters2025generalization}'s paradox.
  Stated as a test of a published unification rather than as a new distinction
  (\S\ref{sec:collapse}); a null here is evidence \emph{for} that unification and is reported as
  such.
\item \textbf{C4} (\S\ref{sec:diagnostic}, \S\ref{sec:prevalence}). Validation power, defined
  relative to an \emph{input mode} (Definition~\ref{def:vp}, Corollary~\ref{cor:modes}), and
  measured on \vnv{c4SetsAudited}\ public sets in both modes. The mode is load-bearing: the same
  sets read as sound per item and as \vnv{c4PairMin}--\vnv{c4PairMax}\ recoverable paired.
  \emph{Needs no annotation of ours.} The per-axis form of the prediction is not testable on
  these sets and remains open.
\item \textbf{C5} (\S\ref{sec:consequence}). Against a blinded per-sentence replacement set, the
  inflation $\Delta$ in judge accuracy from validating on the inherited set.
\item \textbf{C6} (\S\ref{sec:checklist}). Released interventions, human direction labels,
  diagnostic, and checklist.
\end{itemize}

\noindent\textbf{What is not claimed.} The invariance / directional-expectation pair is
\citet{ribeiro2020checklist}'s, named for task models; we claim its systematic application to
evaluators. Shortcut exploitation is a mature literature on the model side
(\S\ref{sec:related}), and our control arm is a replication that exists to make the sensitivity
arm interpretable. \scopeax\ and \strengthax\ do not exhaust overreach; they are the two that can
be edited minimally and verified by humans. The two axes have been noticed before, twice, and both
times deliberately collapsed: we claim that the collapse costs something measurable, which is
narrower and more checkable than novelty (\S\ref{sec:collapse}). A judge insensitive on
\strengthax\ may be perfectly serviceable where strength is not the construct, and human labels
are not ground truth in general \citep{jacobs2021measurement}.

\begin{warnbox}
\noindent\textbf{Companion submission, disclosed.} A second submission by an overlapping author
set characterises the corpus that supplies our probe items and C4's case study
\citep{ecpcompanion}. It adopts this paper's decomposition and its construct-blind predictor
family, and reports the corpus-side measurements; this paper reports the judge-side ones. Neither
depends on the other's results, and the one quantity that would cross between them, a bound on
judge accuracy from a per-sentence gold set, is unrun in both and stated as such in both
(\S\ref{sec:consequence}).
\end{warnbox}

\begin{warnbox}
\noindent\textbf{Anti-circularity.}
An intervention belongs to the sensitivity arm only if it changes the correct verdict, and if a
model decides that, then the experiment tests a model against itself. Direction is therefore set
by human annotators, at agreement $\ge \vnv{humanAgreementMin}$, and generation, verification,
and judging are three
disjoint model families. \S\ref{sec:direction} gives the protocol and reports what it discarded.
\end{warnbox}

\section{Construct validity for an evaluator}
\label{sec:validity}

\begin{figure}[t]
\centering
\includegraphics[width=\textwidth]{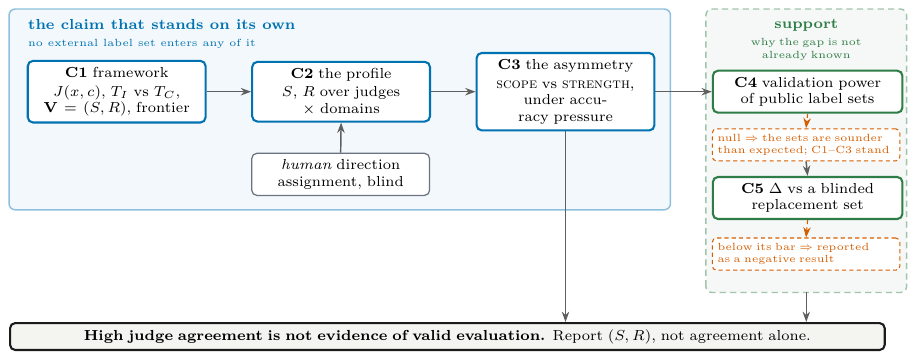}
\caption{The paper's claim structure, in two tiers. The left group is the claim that stands on
its own: no external label set enters C1, C2 or C3, and direction is assigned by humans under
Protocol~\ref{prot:direction}. The right group is support, and each of its stages carries the
consequence of its own null result in the dashed box beneath it. The independence is the
figure's shape rather than an assertion in this caption, and it is why the premise cannot
evaporate on one experiment (\S\ref{sec:whatisclaimed}).}
\label{fig:arch}
\end{figure}

\subsection{The object, and the two conditions}
\label{sec:object}

The vocabulary is borrowed and we use it in its original sense. \emph{Construct validity} is
the question of whether an instrument measures the theoretical construct it is claimed to
measure, as opposed to correlating with it \citep{cronbach1955construct}; the modern treatment
makes validity a property of an \emph{interpretation of scores} rather than of the instrument in
isolation \citep{messick1989validity}, and the measurement-theoretic reading insists that an
instrument's validity is a claim about a causal relation between the attribute and the score
\citep{borsboom2005measuring}. \citet{jacobs2021measurement} bring this vocabulary into machine
learning. What has been missing on the evaluator side is not the vocabulary but an operational
test, and the rest of this section is one.

Write an evaluator as a function
\begin{equation}
\Jud: \mathcal{X} \times \mathcal{C} \longrightarrow \mathcal{Y},
\qquad \Jud(x,\Crit) = y,
\label{eq:judge}
\end{equation}
mapping an item $x$ and an evaluation criterion $\Crit$ to a verdict $y$ in a finite, ordered
verdict space $\mathcal{Y}$. Fix $\Crit$ throughout: the question is not whether a judge
understands the criterion it was given, but how it responds to changes in $x$ under that
criterion. Let $y^{\star}(x,\Crit)$ denote the correct verdict, which for the constructs we
study is a fact about the relation between the claim and its evidence rather than about any
annotator's opinion, though establishing it in practice requires annotators, and
\S\ref{sec:direction} is about that gap.

An intervention is a map $T: \mathcal{X} \to \mathcal{X}$. Partition the interventions of
interest by their effect on the \emph{correct} verdict, not on the surface of $x$.

\begin{definition}[Intervention classes]
\label{def:classes}
$T$ is \emph{construct-preserving} for $\Crit$, written $T \in \Tinv$, if
$y^{\star}(T x,\Crit) = y^{\star}(x,\Crit)$ for all $x$ in the domain of interest. It is
\emph{construct-changing}, $T \in \Tcon$, if $y^{\star}(T x,\Crit) \neq y^{\star}(x,\Crit)$.
\end{definition}

\noindent Three consequences of defining the classes this way, all of which matter downstream.
Membership depends on $\Crit$: register normalisation is in $\Tinv$ when the criterion is claim
calibration and in $\Tcon$ when the criterion is tone. Membership is unrelated to edit
magnitude: a two-word quantifier change is in $\Tcon$ while a rewritten paragraph can be in
$\Tinv$, so the classes cannot be approximated by an edit-distance threshold. And membership is
not observable from $x$ alone; it is a claim about $y^{\star}$, which is why
Protocol~\ref{prot:direction} exists and why a pipeline that assigns it by model output is
circular.

\begin{definition}[Construct validity of an evaluator]
\label{def:validity}
$\Jud$ is \emph{construct-valid} for $\Crit$ on $(\Tinv,\Tcon)$ if
\begin{align}
\Jud(x,\Crit) &= \Jud(T x,\Crit) &&\text{for all } T \in \Tinv, \label{eq:inv}\\
\Jud(x,\Crit) &\neq \Jud(T x,\Crit) &&\text{for all } T \in \Tcon. \label{eq:sens}
\end{align}
\end{definition}

\begin{figure}[t]
\centering
\includegraphics[width=\textwidth]{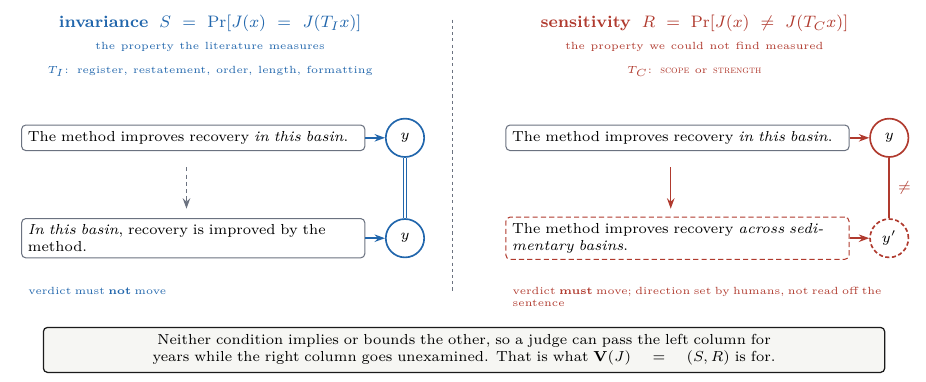}
\caption{The two conditions of Definition~\ref{def:validity} on one claim. Left: the property
the judge literature measures: an edit that changes register, restatement, clause order, length
or formatting must not move the verdict. Hedging is deliberately absent from that list, for the
reason in Appendix~\ref{app:interventions}. Right: the property we could not find measured: an edit that
changes what the claim covers or how hard it commits \emph{must} move it. The right column's
frame is drawn open because the correct verdict there is set by human annotation
(\S\ref{sec:direction}), not read off the sentence.}
\label{fig:twoconditions}
\end{figure}

\noindent Equation~\eqref{eq:inv} is the invariance property that judge-robustness work
measures under many names: consistency, bias, shortcut sensitivity. Equation~\eqref{eq:sens}
is its complement. Definition~\ref{def:validity} is not new as a \emph{form}: it is the
invariance / directional-expectation pair of \citet{ribeiro2020checklist}, stated for an
evaluator rather than for a task model. What is new is measuring the second half for evaluators
and finding that it does not follow from the first.

\begin{definition}[Validity profile]
\label{def:profile}
For a judge $\Jud$, an item distribution $\mathcal{D}$ over $\mathcal{X}$, and intervention
families $\Tinv, \Tcon$,
\begin{equation}
\Sinv(\Jud) \;=\; \Pr_{x\sim\mathcal{D},\,T\sim\Tinv}\!\big[\Jud(x)=\Jud(Tx)\big],
\qquad
\Rsens(\Jud) \;=\; \Pr_{x\sim\mathcal{D},\,T\sim\Tcon}\!\big[\Jud(x)\neq\Jud(Tx)\big],
\label{eq:SR}
\end{equation}
and the \emph{validity profile} is $\Vprof(\Jud) = (\Sinv,\Rsens) \in [0,1]^2$. Both are
properties of the triple (judge, distribution, family): a profile is reported with its families
named or it is not interpretable.
\end{definition}

\subsection{Why the profile is not a scalar}
\label{sec:noscalar}

A reviewer will ask for one number. The answer is that no scalar summary of
Eq.~\eqref{eq:SR} can distinguish a valid judge from a broken one, and the reason is structural
rather than a matter of losing nuance.

\begin{proposition}[The coordinates are independent]
\label{prop:indep}
For every $(s,r) \in [0,1]^2$ there exists a judge attaining $\Vprof = (s,r)$. In particular
neither coordinate constrains the other, and no inequality between them holds in general.
\end{proposition}

\begin{proof}
Fix $\mathcal{Y} \supseteq \{0,1\}$. Construct $\Jud$ by independent randomisation on the two
families: on an input reached through $T \in \Tinv$, return the verdict assigned to the base
item with probability $s$ and the other verdict otherwise; on an input reached through
$T \in \Tcon$, return the other verdict with probability $r$ and the base verdict otherwise.
The two rules act on disjoint sets of (base, intervention) pairs, so the two probabilities in
Eq.~\eqref{eq:SR} are set independently and equal $s$ and $r$ respectively. Since $(s,r)$ was
arbitrary, the whole unit square is attained.
\end{proof}

\noindent Proposition~\ref{prop:indep} is elementary, and that is the point: it says the
literature's decade of invariance results places \emph{no} bound whatever on $\Rsens$. A field
could drive $\Sinv$ to $1$ across every known surface property and learn nothing about whether
its instruments track the construct.

\begin{proposition}[Both coordinates are individually gameable, in opposite directions]
\label{prop:degenerate}
Let $\Jud_{\mathrm{const}}(x,\Crit) = y_0$ for a fixed $y_0$, and let $\Jud_{\mathrm{flip}}$
return a verdict drawn to differ from the base verdict whenever the input was produced by any
intervention. Then $\Vprof(\Jud_{\mathrm{const}}) = (1,0)$ and
$\Vprof(\Jud_{\mathrm{flip}}) = (0,1)$. Neither reads $x$ in a way that depends on $\Crit$.
\end{proposition}

\begin{proof}
Immediate from Eq.~\eqref{eq:SR}: $\Jud_{\mathrm{const}}$ satisfies
$\Jud(x)=\Jud(Tx)$ identically, giving $\Sinv=1$, and never satisfies $\Jud(x)\neq\Jud(Tx)$,
giving $\Rsens=0$. Symmetrically for $\Jud_{\mathrm{flip}}$.
\end{proof}

\begin{corollary}[No monotone scalar separates valid from degenerate]
\label{cor:scalar}
Let $g:[0,1]^2 \to \mathbb{R}$ be any function non-decreasing in each argument, and suppose
$g$ assigns a passing score to some judge with $\Vprof = (s^{\ast},r^{\ast})$,
$s^{\ast},r^{\ast} < 1$. If $g$ is symmetric (as a weighted mean with $\lambda=\tfrac12$, or
a harmonic mean, is) then $g(1,0) = g(0,1)$, and any threshold that admits one degenerate
judge admits the other. More generally, for $\lambda\in(0,1)$ the weighted mean satisfies
$\lambda \cdot 1 + (1-\lambda)\cdot 0 = \lambda$, so a threshold below $\max(\lambda,1-\lambda)$
admits a judge that never reads its input.
\end{corollary}

\noindent The operational rule we adopt follows directly, and it is a commitment rather than a
preference: \textbf{$\Rsens$ is never reported without $\Sinv$ measured on the same judge over
matched interventions}, and a judge whose $\Sinv$ falls below a pre-registered floor is reported
as \emph{uninterpretable} rather than as sensitive.

\begin{figure}[tb]
\centering
\includegraphics[width=0.62\textwidth]{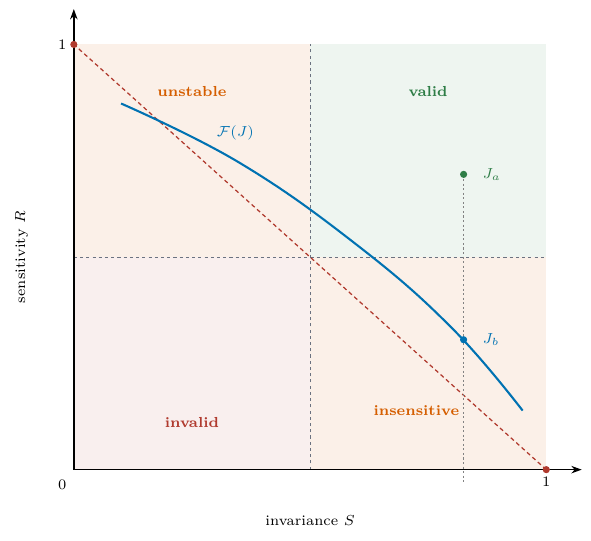}
\caption{The $(\Sinv,\Rsens)$ plane. The two degenerate instruments of
Proposition~\ref{prop:degenerate} are the marked corners: a judge that never moves sits at
$(1,0)$, and one that flips on anything sits at $(0,1)$. The dashed line joining them is the
level set on which a symmetric scalar summary scores them identically
(Corollary~\ref{cor:scalar}), which is why this paper reports the pair. $\mathcal{F}(J)$ traces
one judge as $\tau$ varies, and its default $\tau$, $J_b$, is the only point usually reported.
$J_a$ sits at the same invariance as $J_b$ and in a different quadrant, with the projection to
the axis drawn: what invariance testing alone resolves is that one coordinate, which does not
separate \emph{valid} from \emph{insensitive}. The corners are glossed in
Table~\ref{tab:corners}.}
\label{fig:plane}
\end{figure}

\begin{table}[t]
\centering
\caption{The four corners of the profile, and what each is. The two cells in the top row are
the ones that matter: invariance testing alone cannot separate them, because it does not measure
the coordinate on which they differ.}
\label{tab:corners}
\small
\begin{tabular}{@{}l p{53mm} p{53mm}@{}}
\toprule
 & $\Rsens$ high & $\Rsens$ low \\
\midrule
$\Sinv$ high & \textbf{valid} & \textbf{insensitive}: reliable, and measuring something else \\
$\Sinv$ low  & unstable: may be flipping indiscriminately & \textbf{invalid} \\
\bottomrule
\end{tabular}
\end{table}

\subsection{The frontier, and what a fair comparison is}
\label{sec:frontier}

A judge exposing a decision threshold can trade one coordinate for the other, so a single
profile is one point on a curve and two judges compared at their defaults are two arbitrary
points on two different curves.

\begin{definition}[Validity frontier]
\label{def:frontier}
For a judge family $\{\Jud_\tau\}$ indexed by a decision threshold $\tau$,
\begin{equation}
\mathcal{F}(\Jud) \;=\; \big\{\,\big(\Sinv(\Jud_\tau),\,\Rsens(\Jud_\tau)\big) \;:\; \tau \,\big\}
\;\subset\; [0,1]^2 .
\end{equation}
\end{definition}

\begin{proposition}[Threshold comparisons are not validity comparisons]
\label{prop:frontier}
There exist judges $\Jud^{(1)},\Jud^{(2)}$ and thresholds $\tau_1,\tau_2$ such that
$\Rsens(\Jud^{(1)}_{\tau_1}) > \Rsens(\Jud^{(2)}_{\tau_2})$ while
$\mathcal{F}(\Jud^{(2)})$ dominates $\mathcal{F}(\Jud^{(1)})$ pointwise; that is, the judge
with the higher reported sensitivity is the strictly worse instrument.
\end{proposition}

\begin{proof}
Take $\mathcal{F}(\Jud^{(1)}) = \{(1-r,\,r)\}$ and
$\mathcal{F}(\Jud^{(2)}) = \{(1-r+\epsilon,\,r)\}$ for $r\in[0,1-\epsilon]$, so $\Jud^{(2)}$
attains strictly greater $\Sinv$ at every $\Rsens$. Choosing $\tau_1$ with $r=0.9$ and $\tau_2$
with $r=0.1$ gives $\Rsens(\Jud^{(1)}_{\tau_1}) = 0.9 > 0.1$, while $\Jud^{(2)}$ dominates.
\end{proof}

\noindent Hence: judges are compared by frontier, or at a matched $\Sinv$, and never by raw
$\Rsens$ at default thresholds. This is the same discipline as reporting an ROC curve rather
than one accuracy, for the same reason.

\begin{remark}[Getting a frontier out of a judge that has no visible parameters]
\label{rem:graded}
A judge behind an API exposes no decision threshold, no logits one can rely on across vendors,
and no internal state. It would seem to follow that most judges admit only a single point, and
that Definition~\ref{def:frontier} is unavailable for them. It does not follow, and the fix is a
design choice rather than an approximation: \textbf{elicit a graded verdict and threshold it
ourselves.} A judge asked for a calibration score on a fixed scale, together with a rule mapping
scores to verdicts, \emph{is} a family $\{\Jud_\tau\}$, where $\tau$ is the cut we apply to its
own output and control exactly. Every judge then has a frontier on identical terms, whatever
is or is not visible inside it.

Two things this buys beyond convenience. Comparisons become uniform, so no row of
Table~\ref{tab:judges} is a weaker claim than another for reasons of vendor access. And the
threshold is \emph{ours}, so it cannot be tuned per judge to flatter a result: one rule, applied
to all, stated in Appendix~\ref{app:judges}. What it costs is that the graded scale is part of
the instrument now: a judge that grades coarsely has a coarser frontier, and
Appendix~\ref{app:secondary} S3 reports how much of the measured insensitivity is attributable
to scale granularity rather than to the judge.
\end{remark}

\noindent We therefore do not manufacture a frontier by varying a prompt and present it as one;
we obtain one by asking for a score and cutting it. A judge that refuses to grade, or grades
degenerately, is reported as a single point with that fact stated.

\begin{remark}[Relation to reliability, and why fixing one does not fix the other]
\label{rem:reliability}
Test--retest consistency and human agreement are properties of $\Jud$ at fixed input;
$\Sinv$ extends them to perturbed input. By Proposition~\ref{prop:indep} none of the three
constrains $\Rsens$. So the finding that judges are unreliable \citep{reliability2026} and the
finding that they are insensitive are independent facts about the same instrument, and an
intervention that raises agreement (ensembling, self-consistency, rubric refinement
\citep{gunjal2025rubrics}) has no predicted effect on $\Rsens$ at all. \S\ref{sec:probes} reports whether that prediction
holds.
\end{remark}

\section{Related work}
\label{sec:related}

We read the nearby literature for what each result had to assume rather than for what it found,
because a result can be correct and replicated while resting on an assumption that fails at the
limit, and the assumption is what a new paper can move.

Three recur, and this paper tests all three. \textbf{That reliability is the binding constraint on
judge quality}, assumed by the invariance literature and false by Proposition~\ref{prop:indep},
which shows the two coordinates place no bound on each other. \textbf{That overclaiming is one
thing}, assumed by the works closest to ours, which reach the scope/strength distinction and then
unify it deliberately, theoretically \citep{jiang2025conformal} and empirically
\citep{rigourate2026}; we treat the unification as a hypothesis with a refutation condition.
\textbf{That agreement with a validated label set is evidence of validity}, assumed nearly
everywhere and the subject of \S\ref{sec:consequence}, where a construct-blind family reproduces
\vnv{c4PairMin}--\vnv{c4PairMax} of five public sets' labels.

Four lines of work assume nothing this paper tests, and are cited as ancestry rather than as
contrast. Behavioural testing supplies the \emph{invariance} versus \emph{directional
expectation} distinction we adopt \citep{ribeiro2020checklist}; judge evaluation took up the first
and not the second, so what we claim is the application to evaluators, not the distinction.
Validity position papers make the call we answer \citep{dietz2025tropes}. NLI artifacts and
contrast sets are the methodological ancestor, benchmarks solvable without doing the task
\citep{gururangan2018artifacts,poliak2018hypothesis,gardner2020contrast}; the contribution here is
the instrument, the decomposition and the audit rather than the idea. Measurement modelling
supplies the vocabulary \citep{jacobs2021measurement}, to which we add an operational test.

Two more are contrasts that fit in a clause. Automatic reliability stress-testing assumes the
benchmark's labels are sound \citep{judgeharness2026}, which is exactly what our harness tests, so
the two compose rather than compete. And positional bias of faithfulness locates its defect at
predictable positions in the output \citep{wan2024positional}, where mis-scoping has no position
at all: it is a relation between a sentence and its evidence.

Table~\ref{tab:related} states each remaining assumption against the result that carries it; the
right-hand column is where we agree with a result and disagree with what it is taken to license.

%
{\footnotesize
\setlength{\tabcolsep}{4pt}
\renewcommand{\arraystretch}{1.05}
\begin{longtable}{@{}p{0.19\textwidth}p{0.37\textwidth}p{0.36\textwidth}@{}}
\caption{What the nearby results establish, and the assumption each one needs that we test.}
\label{tab:related}\\
\toprule
work & establishes & assumes, and we test \\
\midrule
\endfirsthead
\toprule
work & establishes & assumes, and we test \\
\midrule
\endhead
\bottomrule
\endlastfoot
Judge reliability audits \citep{reliability2026} &
judges are unreliable, and exact-match agreement overstates discrimination &
that the label sets are ground truth. Judges are measured against fixed sets, never audited \\
Mechanistic accounts of judge bias \citep{unfairjudge2026} &
surface manipulations move judge scores, visibly in activations &
only \emph{specificity}: edits that should not flip a verdict. Sensitivity is untested, and no
bias type there is scope or quantifier \\
Noise-corrected evaluation \citep{noisyvalid2026} &
TPR/FPR estimated on a calibration set can be corrected for &
that the calibration set is sound. The guarantees inherit whatever it leaks \\
Judge input-shortcut probes \citep{silentjudge2025} &
judges follow injected metadata cues and never acknowledge them &
that the shortcut is in reading the input. Perturbs inputs, not the construct, and audits no
label set \\
Automatic factual-consistency metrics \citep{zha2023alignscore} &
claim support can be scored by one alignment function &
that the axis of interest is support. A mis-scoped claim is not unsupported, so such a metric is
silent by construction, not by weakness \\
Surface and style confounds, in judges and in reward models
\citep{zheng2023judging,wang2023unfair,zhang2025formatbias,huang2025posthoc,%
lambert2024rewardbench,malik2025rewardbench2,liu2024rmbench,chen2024odin,prism2025,rmbiases2026} &
verbosity, order, formatting and length shift verdicts, and preference labels are recoverable
from them, well enough that benchmarks are now \emph{designed} against it &
that the property should not have mattered, and that the confound is model-side. Every one is an
invariance test; the label-set side has no headroom metric, no cross-set audit, no reporting
convention \\
Preference optimisation \citep{rafailov2023dpo,skalse2022reward} &
a preference signal distils into a policy, and a proxy reward is gameable in the limit &
that the signal means what it says. What a judge cannot distinguish is what the policy may
exploit, so $\Rsens$ is a property of the training signal, not only of the report \\
\end{longtable}}

\subsection{The decomposition, seen twice and collapsed twice}
\label{sec:collapse}

The claim that nobody separates \scopeax\ from \strengthax\ would be false. Two published works
reach both mechanisms and unify them deliberately, which is a stronger position than a gap claim:
a collapse has authors who chose it and a cost that can be measured.

\paragraph{Collapsed theoretically.} \citet{jiang2025conformal} reinterpret linguistic calibration
as answer-set prediction, under which widening what a claim covers and hedging how hard it commits
are the same operation: both enlarge the set of possible worlds in which the claim holds. The
unification is what buys the conformal guarantee, since a single nested family of answer sets is
what a calibrated set predictor needs, and for generation under a coverage guarantee it is the
right modelling choice. Our claim is about evaluation, where \S\ref{sec:noscalar}'s argument
applies to their collapse as much as to a scalar $\Vprof$: a judge that misses a hedge deletion
and catches a quantifier widening has a property one answer-set coordinate cannot express.

\paragraph{Collapsed empirically, in the nearest work's own appendix.} \citet{rigourate2026}
score overstatement as one continuous value, but their appendix decomposes linguistic certainty
into extent, number, framing and probability, and reports that increased overstatement is driven
by greater use of \textbf{probability} and \textbf{extent} certainty. Those are our two axes under
other names, found in their data and discarded by the headline metric. It is the strongest external
support the decomposition has: two mechanisms were seen driving one score, and nobody asked whether
an evaluator responds to them differently.

\paragraph{Observed together outside our domain.} Studying whether a chatbot should assert
generics about social groups, \citet{zhu2026generics} reports that models inconsistently hedge and
decline generalisations reflecting documented patterns. Generic formation is the \scopeax\
operation and hedging the \strengthax\ one, and a model turning both inconsistently is not what a
single unstructured dial produces. The evidence is convergent from a domain with no stake in our
construct.

This makes C3 a test with a named opposing prediction rather than a gap-filling exercise: had
\scopeax\ and \strengthax\ sensitivity not differed at matched $\Sinv$, that would have been
evidence \emph{for} the collapse.

\paragraph{On the sensitivity arm, the same correction.} Applying edits intended to change a
verdict is not unprecedented. \citet{park2026perceptual} build minimally edited counterfactual
responses isolating perceptual errors, catching a judge that fails to downgrade them. Three things
separate it from \S\ref{sec:probes}: it is multimodal and about visual-versus-textual conflict; the
intended direction is established by construction rather than by human adjudication, which is the
circularity Protocol~\ref{prot:direction} closes; and its object is mitigation via reward
modelling rather than measurement, so it reports no invariance arm and therefore no profile. The
claim we keep is the pair, not the direction: $\Sinv$ and $\Rsens$ together, for text constructs,
with direction set by people.

\paragraph{The framing is converging.} \citet{roy2026premise} treat judge rubrics as measurement
specifications and audit them for structural adequacy, reliability, preference fit and adversarial
robustness, concluding that no source is simultaneously reliable, preference-predictive and robust,
and that high inter-rater agreement does not prevent exploitability. That is this paper's premise
reached from the rubric side by different authors. Where it stops is where
\S\ref{sec:diagnostic} starts: PReMISE audits the rubric and the judge, not the \emph{label set}
the judge is scored against, and Definition~\ref{def:vp} is the quantity for that set.

Table~\ref{tab:related} is the positioning; three remarks it cannot carry follow.

\paragraph{The wider ancestry.} That a benchmark can be solvable without doing the task is a
general phenomenon: shortcut learning in deep networks \citep{geirhos2020shortcut}, syntactic
heuristics passing inference benchmarks \citep{mccoy2019hans}, annotator identity recoverable from
labels \citep{geva2019annotator}, counterfactually augmented data proposed against such
correlations \citep{kaushik2020counterfactual}, the benchmarking critique that a leaderboard number
can be uninformative about the capability it names
\citep{bowman2021fixdatasets,raji2021benchmark}, and a formal treatment of reward hacking
\citep{skalse2022reward}. Our contribution is not to observe that this can happen to evaluators but
to give the test that says whether it has.

\paragraph{The nearest work.} \citet{jiang2025conformal} formalise a factuality/specificity
trade-off and rewrite claims to trade one against the other. The distinction is the direction of
the question: they ask how a \emph{generator} should set the specificity of what it writes, we ask
whether an \emph{evaluator} notices when specificity has been changed for it. A calibration result
about generation places no constraint on evaluator sensitivity. \citet{claimcheck2025} is adjacent
in material rather than in question, asking how well grounded LLM critiques of papers are.

\paragraph{Which constructs are at risk, and an adjacent negative result.} The constructs most
exposed are those whose labels are expensive and whose surface correlates are cheap. Claim
specificity is the case we develop, because a resource exists whose labels are known to be
length-separable \citep{chen2026tridehall} and because the construct is well posed
\citep{jiang2025conformal,peters2025generalization}. Probing work on a neighbouring construct, the
strength of evidence supporting a clinical claim, recovered a graded label linearly in every model
tested and then found the recoverable signal was largely lexical and did not transfer across
topics \citep{arasteh2026evidence}. That is our shape of finding from the representation side, and
it is why \S\ref{sec:prevalence} reports transfer rather than only within-set recovery.

\section{Method: the interventions, the direction protocol, and the diagnostic}
\label{sec:probes}

This section fixes the instrument before any result is read off it: how the two
intervention classes of Definition~\ref{def:classes} are realised as concrete edits, who
decides that an edit belongs in the sensitivity arm, which judges and domains the profile
is measured over, and (\S\ref{sec:diagnostic}) the second instrument the paper needs,
which measures not a judge but the label set a judge is validated against. Every choice here
is frozen in the repository before the first judge call, and \S\ref{sec:direction} is the
one a reader should push hardest on.

\subsection{Instantiating the two intervention classes}

The construct is whether a claim overreaches its evidence. $\Tcon$ splits into two axes, fixed in
\texttt{probes/\allowbreak edit\_\allowbreak taxonomy.py} before any judge is run:

\begin{description}[leftmargin=1.4em,itemsep=1pt,topsep=2pt]
\item[\scopeax] (referent set grows, commitment fixed): \emph{population widening},
  \emph{domain extension}, \emph{tense/modality shift} (one observation restated as a standing
  property) and \emph{quantifier strengthening}.
\item[\strengthax] (referent set fixed, commitment rises): \emph{condition removal} (deleting
  the circumstances the effect was observed under) \emph{hedge removal}, and
  \emph{intensifier addition}.
\end{description}

\noindent Both axes are linguistic distinctions before they are machine-learning ones. Hedging
is a studied feature of scientific writing with a documented rhetorical function
\citep{hyland1998hedging}, and uncertainty together with \emph{its scope} has been annotated at
corpus scale in biomedical text \citep{vincze2008bioscope}, which is direct evidence that the
two things we separate are separately annotatable by humans. The adjacent verification
literature asks whether a claim is \emph{supported} \citep{thorne2018fever,wadden2020scifact};
our construct begins where that one ends, since every item in the $\Tcon$ arm is supported on
both sides of the pair and differs only in what it covers or how hard it commits.

\noindent \textbf{And \strengthax\ is a manipulable quantity on the model side, not only a
rhetorical one}, which matters because an axis no model represents is an axis no judge could
plausibly be asked to track. Linguistic calibration (making an agent's expressed confidence
match its competence) has been trained for directly \citep{mielke2022reducing}; expressed
uncertainty has since been located as a \emph{linear} feature in representation space and steered
along it \citep{ji2025verbal}; and \citet{arasteh2026evidence} recover a graded evidence-strength
label linearly in every model tested. Our claim is not that models cannot represent commitment.
\emph{It is the sharper and more awkward one: commitment is representable, steerable and linearly
decodable, and judges still do not act on it.} The \scopeax\ axis has the parallel grounding on the
generation side, where decontextualising a scientific snippet (lifting it off the source without
widening what it covers) is a defined task with gold data \citep{newman2023decontext}, and
over-generalisation under that pressure is measured at scale \citep{peters2025generalization}.

\noindent $\Tinv$ is the control arm, and it deliberately covers the surface properties the judge
literature has \emph{already} shown to move verdicts, because a low flip rate on properties nobody
has implicated proves nothing: \emph{hedge padding}, \emph{same-scope elaboration},
\emph{register shift}, \emph{clause reordering}, \emph{verbosity padding}
\citep{zheng2023judging}, and \emph{format shift} \citep{zhang2025formatbias}.

\begin{figure}[t]
\centering
\includegraphics[width=0.9\textwidth]{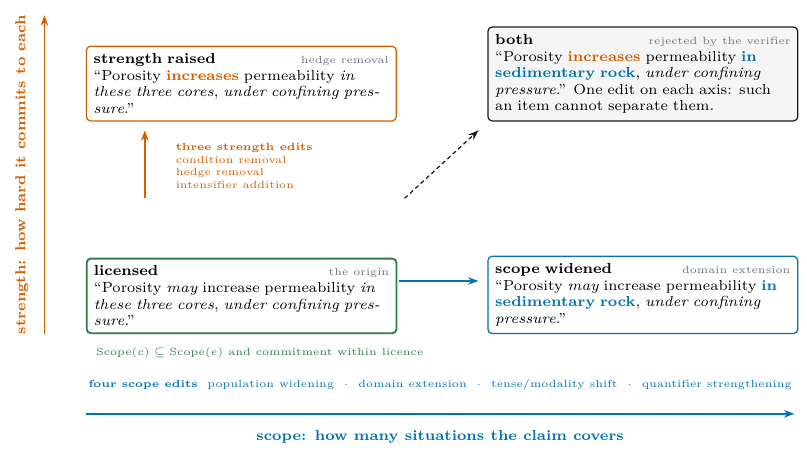}
\caption{The decomposition on one worked claim. From a single licensed sentence, \scopeax\ moves
right (the referent set grows while commitment is untouched) and \strengthax\ moves up: the
referent set is fixed and commitment rises. The seven $\Tcon$ edit types sit beside the axis each
moves, and in each cell the changed span is coloured with the axis that moved it; every cell is
one edit from the origin. The diagonal is the case the verifier rejects: an item that moved both
axes cannot separate them, and admitting it would contaminate the very contrast
\S\ref{sec:res-axes} rests on.}
\label{fig:axes-diagram}
\end{figure}

\begin{remark}[Why these two axes and not one]
The axes are separable in the sense measurement requires: an edit on one can be applied while
holding the other's operators fixed, and a verifier can check which moved. They are not claimed to
exhaust overreach. Splitting them is a modelling choice that earns its place only if judges behave
differently across it, which is a result, reported below, not an assumption.
\end{remark}

\subsection{Who decides that the verdict should change}
\label{sec:direction}

An intervention is in $\Tcon$ only if it changes the correct verdict, which makes this the
paper's load-bearing methodology. If a language model adjudicates that, we are testing models
against models and any asymmetry could be a shared blind spot rather than a property of judges.

\begin{protocol}[Direction assignment]
\label{prot:direction}
Three disjoint model families, and humans hold the deciding vote.
\begin{enumerate}[leftmargin=1.4em,itemsep=1pt,topsep=2pt]
\item \textbf{Generate} the candidate edit (family A).
\item \textbf{Verify} mechanically (family B, never A): facts unchanged; the intended axis moved
  and the other axis's operators untouched; no provenance anchor introduced; register comparable;
  length within tolerance.
\item \textbf{Assign direction by human annotation.} Annotators see the pair without knowing which
  member is the edit, which axis it belongs to, or what any model said, and answer only whether
  the second sentence claims more than the evidence licenses. An item enters the sensitivity arm
  only at agreement $\ge \vnv{humanAgreementMin}$ (pre-set;
  \texttt{HUMAN\_\allowbreak DIRECTION\_\allowbreak AGREEMENT\_\allowbreak MIN}).
\item \textbf{Judge} (family C, never A or B).
\end{enumerate}
\end{protocol}

\noindent Two reporting rules follow. Per-slot yield appears in the same table as the result,
because if \strengthax\ edits are harder to construct cleanly than \scopeax\ ones then the
surviving \strengthax\ items passed a stricter filter and part of any asymmetry is the pipeline
rather than the judges. And human--verifier disagreement is reported rather than resolved
silently, since it bounds how far the mechanical gate can be trusted where humans were not asked.
The pool and the pre-set bar of \vnv{humanAgreementMin} are specified in
Appendix~\ref{app:protocol}, frozen before any data was collected, together with the rule attached
to it: a slot or axis below the bar is reported as \textbf{not measurable} and excluded from
$\Rsens$ rather than reported with a wider interval, because a wide interval around an ill-defined
construct still asserts the construct exists.

\paragraph{What the annotators found.} \vnv{annRaters} annotators independently labelled all
\vnv{annItems} items; the third worked the full batch rather than arbitrating the first two's
disagreements, so the three label sets are independent. Raw pairwise agreement is \vnv{annRaw}
(Fleiss' $\kappa$ \vnv{annKappa}, the three pairwise rates spanning
\vnv{annPairMin}--\vnv{annPairMax}), and every axis clears the bar: \scopeax\
\vnv{annScopeRaw} ($n=\vnv{annScopeN}$), \strengthax\ \vnv{annStrengthRaw}
($n=\vnv{annStrengthN}$), invariance \vnv{annControlRaw} ($n=\vnv{annControlN}$). Only
\vnv{annNoMajority} items (\vnv{annNoMajorityPct}) drew three different answers.
\textbf{The two construct axes are equally well defined for human readers.} \scopeax\ is
nominally higher, but by less than the spread between annotator pairs, so we claim no ordering;
what the bar was set to rule out, that one of the two distinctions is not one people reliably
make, happened on neither.

Of the annotated items, \vnv{annUsableConstruct} of \vnv{annConstructAnnotated} construct items
carry a majority direction matching the intended one (\vnv{annUsableScope} \scopeax,
\vnv{annUsableStrength} \strengthax) and \vnv{annUsableControl} of \vnv{annControlAnnotated}
invariance items survive; Table~\ref{tab:annotation} gives the per-slot breakdown, with agreement
and yield in separate columns because they fail differently.
\texttt{tense\_\allowbreak modality\_\allowbreak shift} is the lowest-yield slot at \vnv{annTenseYield} on ordinary
agreement: the readers agree, and what they agree on is that the edit often left the claim where
it was. That is a generator yield problem rather than a construct problem, so the slot stays in
\scopeax\ with its yield stated.

\input{figs/tab_annotation_float}

\paragraph{Two invariance slots did not survive.} The two that manipulate length,
\texttt{verbosity\_\allowbreak padding} and \texttt{elaboration}, fall below the bar at
\vnv{annSlotVerbosityPaddingRaw} and \vnv{annSlotElaborationRaw}, while the other
\vnv{annControlPassedSlots} invariance slots agree at $\ge \vnv{annControlPassedMin}$. They are
excluded under the pre-registered rule. Two separate problems were stacked there, and only the
annotators could have separated them.

The first was ours: the generator padded sentences with \emph{as is often observed},
\emph{in practice} and \emph{as is well known}, which read as filler but each assert something the
evidence never stated. On the first batch \vnv{annPadPreDir} of \vnv{annPadPreN} items drew a
majority \emph{directional} verdict. Three revisions of the generation prompt failed to stop it,
and what worked was a fixed marker list checked programmatically before the verifier sees the pair
(\texttt{epistemic\_\allowbreak marker\_\allowbreak change}): \emph{a prompt states an intention, a filter enforces a
property.} Regenerating the same slots under the gate, with the same annotators, took the rate to
\vnv{annPadPostDir} of \vnv{annPadPostN}.

The second is untouched by the gate. Agreement barely moved
(\vnv{annPadPreAgree} to \vnv{annPadPostAgree}) as the contamination cleared, because the
disagreement changed hands: afterwards one annotator still called an appositive gloss directional
(\vnv{annPadPostCallsHi} of \vnv{annPadPostN}) while the other two almost never did
(\vnv{annPadPostCallsLo}), and that annotator agreed with the others at \vnv{annPairMax} on the
main batch. Whether naming an entity more precisely changes what a sentence claims is a question
our codebook never answered, and one competent reader in three answers it the other way. We report
it as a boundary of the invariance construct: \emph{surface-preserving} is well defined for
formatting, register and word order, and not for appositive elaboration
(\S\ref{sec:limitations}).

\subsection{Judges and domains}

The design floor is six judge families spanning vendors, parameter scales, and
reasoning-enabled versus not, over at least four evaluation domains. That floor is not
negotiable downward for a reason of kind rather than of power: a single-domain result is a
result about that domain, and this paper's claim is about the instrument.
Restricting the study to one construct in one domain would make the finding a fact about
geoscience claim-scope rather than about evaluators, so the profile is measured across domains
whose criteria differ in kind, with the scientific-claim domain as the one where minimal edits are
best posed and human verification is cheapest.

\begin{table}[t]
\centering
\caption{Validity profiles. $\Sinv$ from the $\Tinv$ control arm, $\Rsens$ from $\Tcon$ split by
axis. Per-slot $n$ and human confirmation are part of the result, not appendix material:
\S\ref{sec:direction} explains why. \emph{$n$} counts items whose majority human label matched
the intervention's intent and whose slot met the agreement bar: the items that actually enter
the arm. \emph{human dir.} is that count as a fraction of items annotated, so a low value marks a
slot the generator often failed to move rather than one humans could not read: \texttt{tense
modality shift} is the clearest case, at \vnv{annTenseYield} with ordinary agreement.
$^\dagger$ below the pre-set bar of \vnv{humanAgreementMin} and therefore excluded
(\S\ref{sec:direction}). Pipeline discard before annotation was \vnv{pipelineDiscard} overall,
dominated by the length tolerance; it is reported as an aggregate because
\texttt{build\_pairs.py} keys its rejections by reason rather than by slot, and a per-slot figure
we did not record is not one we will estimate.}
\label{tab:profile}
\small
\input{figs/c4_profile}
\end{table}

\subsection{The diagnostic: what a label set could ever have detected}
\label{sec:diagnostic}

\begin{definition}[Provenance-inherited label]
\label{def:pil}
Let each item $x_i$ be produced under a generative condition $\Cgen_i$, which side of a
preference pair it occupied, which model emitted it, which prompt variant produced it. A label
set $\Lset = \{(x_i, y_i)\}$ is \emph{provenance-inherited} if $y_i = g(\Cgen_i)$ for some $g$,
rather than the result of a reading of $x_i$ that is blind to $\Cgen_i$.
\end{definition}

Definition~\ref{def:pil} is about how the label was obtained, not about whether it is correct. A
provenance-inherited label can be perfectly correct and still be useless for validation, and that
is the point: correctness and informativeness come apart here.

The condition is not exotic. It is what happens by default whenever a label set is assembled from
a generation pipeline rather than from independent reading, which is now the common case: surveys
of synthetic-data curation treat the generating condition as a first-class part of the artifact
\citep{long2024syntheticsurvey}, and the metrics proposed for certifying generated data are about
quality and trustworthiness of the \emph{content} rather than about whether the \emph{label} is
reachable without the construct \citep{zhang2026dataauditor}. Validation power is that missing
quantity. It is also why the diagnostic is cheap: $\Cgen_i$ is usually recorded somewhere in the
pipeline, so a set's authors can compute $\vpow$ on their own artifact before releasing it, and
the sixth item of \S\ref{sec:checklist} asks only that they release the field so others can.

\begin{definition}[Validation power, relative to an input mode]
\label{def:vp}
Let $\mathcal{M}$ be an \emph{input mode}: the structure in which an instrument is shown an item.
Two modes matter here. In \textsc{per-item} mode the instrument sees one response and returns a
label. In \textsc{paired} mode it sees two responses to the same prompt and returns which one the
label set prefers. For a label set $\Lset$ and a mode $\mathcal{M}$, let
$\base(\Lset;\mathcal{M})$ be the agreement with $\Lset$ achieved by the best predictor in a
fixed family of \emph{impoverished} predictors \emph{operating in mode $\mathcal{M}$}, and let
$\ceil(\Lset;\mathcal{M})$ be the human ceiling in that mode. Then
\[
  \vpow(\Lset;\mathcal{M}) \;=\; \ceil(\Lset;\mathcal{M}) - \base(\Lset;\mathcal{M}).
\]
\end{definition}

\begin{warnbox}
\noindent\textbf{$\vpow$ is undefined until the mode is fixed.} A per-item predictor can exploit
only an absolute signature (\emph{responses over this length are preferred}); a paired predictor
exploits a relative one (\emph{whichever of these two is shorter is preferred}). Where a set's
surface confound is relative, the per-item estimate understates recoverability by an amount
nothing bounds in advance: \textsc{MT-Bench} gives $\base = \vnv{c4AimpMtbench}$ per item and
\vnv{c4PairSurfMtbench} paired. The mode reported must be the mode the instrument is used in.
\end{warnbox}

\begin{corollary}[The two modes are not orderable in general]
\label{cor:modes}
Neither mode dominates. A set whose label is a function of an absolute property is recoverable
per item and may be near chance paired, if both sides of every pair share that property. A set
whose label is a function of a relative property is the reverse. Hence a single reported $\base$
without its mode is uninterpretable, and reporting the \emph{lower} of the two is not
conservative; it is simply the answer to whichever question was not asked.
\end{corollary}

Validation power is the discriminative room a judge has to earn. Where $\vpow \approx 0$, a
judge's agreement with $\Lset$ is uninformative at any magnitude, because an instrument that
reads only surface form already reaches the ceiling. Where $\vpow$ is large, agreement is
evidence. Reporting $\vpow$ rather than agreement is the single change we ask for.

\paragraph{The impoverished family, frozen.} Three predictors, each denied the construct by
construction:
\begin{enumerate}[leftmargin=1.4em,itemsep=1pt]
\item \textbf{Length only.} Token count of the item and nothing else.
\item \textbf{Surface only.} Character $n$-grams, punctuation and formatting statistics, and
  hedge-marker counts, no content words carrying the construct.
\item \textbf{Provenance-recoverable.} A predictor of $\Cgen_i$ from $x_i$. If $\Cgen$ is
  recoverable and $y = g(\Cgen)$, then $y$ is reachable without the construct, and this predictor
  bounds how much of $\Lset$ any surface-reading instrument can obtain.
\end{enumerate}

\paragraph{What makes it a test rather than a statistic.} Three requirements, all of which the
released harness enforces. Agreement is chance-corrected, so $\vpow$ is not inflated by class
imbalance. Each $\base$ is reported against a label-permutation null, so a small positive
$\vpow$ is distinguishable from noise. And $\ceil$ carries its own interval, because a ceiling
estimated from few double-annotated items can be the larger source of uncertainty in
Definition~\ref{def:vp}.

\begin{remark}[Why not simply audit the judge harder]
A judge audit answers ``does this instrument behave consistently''. It cannot answer ``does
agreement with this label set mean anything'', because that question is about the label set. The
two are independent: \S\ref{sec:intro}'s apparent contradiction is exactly a well-behaved judge
meeting a leaky ruler.
\end{remark}

\section{Experiments: the profile, the asymmetry, and what the rulers could have seen}
\label{sec:experiments}

Five measurements, in the order they depend on each other. The profile
(\S\ref{sec:res-profile}) is the headline and needs only \S\ref{sec:probes}'s instrument. The
axis decomposition (\S\ref{sec:res-axes}) splits it, and the accuracy-prompting manipulation
(\S\ref{sec:paradox}) is the split's out-of-sample test, since it predicts a published
result's sign structure rather than fitting our own. The validation-power sweep
(\S\ref{sec:prevalence}) then asks why the first three findings are not already in the
record, and \S\ref{sec:consequence} prices the answer. Every measured number below is generated
from \texttt{results/} rather than typed into the manuscript, and the two arms that remain unrun
are named as such where they would appear.

\subsection{Judges are far more invariant than sensitive}
\label{sec:res-profile}

\begin{figure}[t]
\centering
\input{figs/plot_profiles}
\caption{Validity profiles for \vnv{profJudges} judges. Frontiers are drawn for the judges that
expose a decision threshold, which under Remark~\ref{rem:graded} is all of them that grade;
any judge that will not grade is a single labelled point, per Proposition~\ref{prop:frontier}. The dashed diagonal is $\Sinv+\Rsens=1$, where a
judge is trading one property for the other rather than possessing both. Curves are the
monotone hull over our threshold sweep: at each invariance we keep the best sensitivity achieved
at that invariance or stricter, so a curve is a frontier rather than a trace of every threshold.
Every vertex is an attained threshold, so a long straight segment is interpolation between two of
them and not a measured path; the ringed vertex is the matched-invariance operating point
Table~\ref{tab:judges} reads off.}
\label{fig:profiles}
\end{figure}
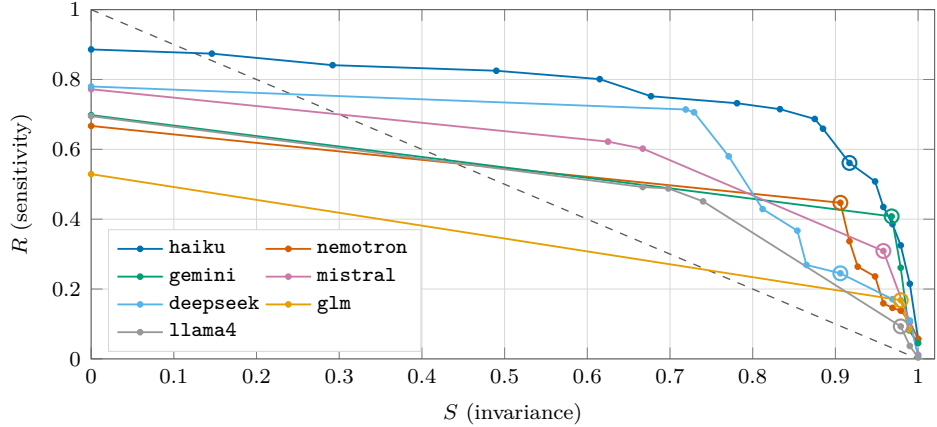

\begin{table}[t]
\centering
\caption{Per-judge profiles. $\Rsens$ is reported per axis and never without $\Sinv$
(Corollary~\ref{cor:scalar}); the final column records whether the row was read off a frontier
or a default threshold, because the two are not comparable. \predkey}
\label{tab:judges}
\footnotesize
\input{figs/tab_judges}
\end{table}

Figure~\ref{fig:profiles} shows the vertical gap the study was built to look for, and it is
large. Held at matched invariance $\Sinv\ge\vnv{profSinvTarget}$, the \vnv{profJudges}\ judges
average $\Sinv=\vnv{profSinvMean}$ against $\Rsens=\vnv{profRsensMean}$: the \emph{insensitive}
cell of Table~\ref{tab:corners}, not the \emph{valid} one. This is not an artifact of a weak
field. The strongest instrument in the sweep, \vnv{profBestJudge}, reaches
$\Rsens=\vnv{profBestRsens}$ and so still misses more than two construct changes in five while
holding its verdict on \vnv{profSinvMean}\ of edits that should not move it.

\paragraph{The comparison that decides Remark~\ref{rem:reliability}.} If ensembling and
self-consistency raise $\Sinv$ without moving $\Rsens$, then the field's standard reliability
remedies do not touch validity, and Proposition~\ref{prop:indep} has an empirical instance
rather than only a proof. We run those two remedies as additional judge rows for this reason.

\subsection{The \scopeax/\strengthax\ asymmetry}
\label{sec:res-axes}

\input{figs/tab_slots_float}
\begin{figure}[t]
\centering
\input{figs/plot_axes}
\caption{Sensitivity per intervention axis, per judge, at matched $\Sinv$, with the control-arm
flip rate on the same scale. The control bar is the reference that makes the other two
interpretable: a construct-edit bar near the control bar is a judge responding to a
construct change no more than to an edit that should not move it at all. Judge labels are
shortened vendor names; per-judge thresholds and instrument quality are in
Table~\ref{tab:judges}.}
\label{fig:axes}
\end{figure}
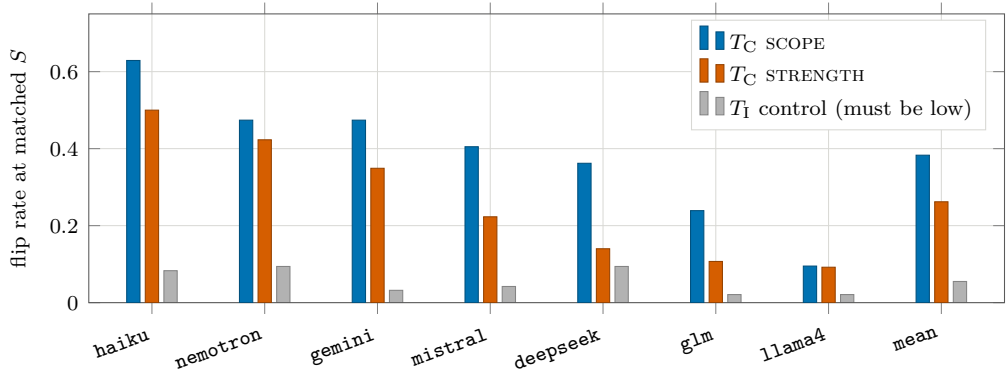

The pattern holds. Compared at matched invariance, $\Rsens^{\scopeax}=\vnv{profScopeMean}$
against $\Rsens^{\strengthax}=\vnv{profStrengthMean}$: a gap of \vnv{gapObs} at a control flip
rate of \vnv{profCtrlFPR}. \textbf{The gap has the same sign on \vnv{profGapPos} of the
\vnv{profGapN} judges}, spanning vendors, parameter scales and reasoning-enabled or not, with
per-judge values in Table~\ref{tab:judges}. Judges detect that a claim now covers more and
largely miss that it now commits harder.

The sign, the magnitude and the condition that would have refuted them were all registered
before measurement (Appendix~\ref{app:predicted}): the gap was to be reported as refuted if it
came in below $0.10$ or failed a paired bootstrap over base items. It clears both. Three further
checks, all specified before the numbers existed, hold. A bootstrap resampling
\emph{base items} rather than pairs (several pairs share a base sentence across slots, so
resampling pairs would treat those as independent draws) puts the gap at $\vnv{gapLo}$ to
$\vnv{gapHi}$ over \vnv{gapBaseItems}\ items, excluding zero. Permuting the axis label within
the construct arm while holding every judge's grades fixed gives \vnv{gapPerm}, with a
largest null gap of \vnv{gapNullMax}. And on the \vnv{gapUnanN}\ items where all
\vnv{annRaters}\ annotators agreed on direction, the pre-committed robustness split, the gap is
\vnv{gapUnan}: \emph{larger} than on the full set, which is the direction that makes it a fact
about judges rather than about annotation difficulty.

\paragraph{The length controls, and why the gap is a lower bound.} Edit magnitude is the
alternative explanation with the most force: \strengthax\ edits are shorter on average than
\scopeax\ ones, so a judge merely less sensitive to small edits would reproduce the asymmetry with
no axis structure. Three checks close it, and the third turns the confound into support.

The magnitude match survives human filtering. Items enter the arm only on human confirmation, and
confirmation rates differ by slot, so the match built at construction had to be re-verified on
what actually entered: median $|\Delta|$ length is \vnv{audMagScope} on \scopeax\ against
\vnv{audMagStrength} on \strengthax. The \emph{signed} change does differ, with \scopeax\ edits
lengthening (\vnv{audSignScope} median, \vnv{audLongerScope} longer) and \strengthax\ edits
shortening (\vnv{audSignStrength}, \vnv{audLongerStrength}). So the second check asks whether
grades track length at all: across every graded sentence, the correlation between a sentence's
length and its grade is \vnv{audCorrMin} to \vnv{audCorrMax} over the \vnv{profJudges}\ judges.
\textbf{It is positive, meaning judges grade longer sentences as slightly better scoped}. Since
the \scopeax\ edit is the longer member, that bias pushes judges toward naming the
\emph{original} on \scopeax\ items and the edit on \strengthax\ ones. It works against the
observed gap, which makes \vnv{gapObs} a lower bound rather than an inflated estimate. Third, the
gap survives stratifying on the sign: \vnv{audGapLong} where the edit lengthened
(\vnv{audGapLongPos} of \vnv{audGapLongN} judges) and \vnv{audGapShort} where it shortened
(\vnv{audGapShortPos} of \vnv{audGapShortN}).

\paragraph{Per slot, which is where the result is most useful.} Table~\ref{tab:slots} gives
$\Rsens$ for each intervention. The ordering is more informative than the axis means:
\texttt{quantifier\_\allowbreak strengthening} and \texttt{domain\_extension} are the most detected at
\vnv{audBestSlotR}, and \texttt{hedging\_removal} the least at \vnv{audWorstSlotR}, a factor of
more than two between two edits that a reader would call equally clear cases of overreach. The
magnitude column shows why magnitude cannot be stratified independently: operation fixes it, since
a deleted condition clause cannot be a small edit and an added intensifier cannot be a large one.

\noindent The slot distributions overlap across the axes: \texttt{\vnv{audCrossSlot}} is a
\strengthax\ slot detected at \vnv{audCrossSlotR}, above \vnv{audCrossBeats}\ of the four
\scopeax\ slots. The axis gap is therefore a difference between the central tendencies of two
heterogeneous groups rather than a categorical separation, and the registered decomposition
accounts for $\eta^2=\vnv{audEtaReg}$ of the variance in slot-level $\Rsens$; \S\ref{sec:limitations}
records an alternative grouping that fits comparably.

\subsection{The accuracy-prompting paradox}
\label{sec:paradox}

If overreach were one quantity, an instruction to be accurate could not increase it. Under
Definition~\ref{def:classes} it can: a demand for accuracy is a demand to sound certain, which
suppresses \scopeax\ while inducing \strengthax, because hedges and stated conditions are exactly
what reads as uncertainty. We rerun the manipulation with the axes scored separately, on the
generator side, and then ask whether judges' per-axis sensitivity predicts which of the two a
model will exploit under that pressure.
\deferred{The two per-axis effects. The registered prediction is that they have opposite signs, with the refutation condition in Appendix~\ref{app:predicted}: if both axes move the same way, the decomposition does not explain the paradox. The manipulation is generator-side, so it needs fresh generation and a per-axis readout that does not yet exist.}

\paragraph{Declared overlap.} The control arm replicates established results:
\citet{unfairjudge2026} on surface manipulation, \citet{silentjudge2025} on injected metadata
cues, \citet{zhang2025formatbias} on format, \citet{zheng2023judging} on verbosity, and
\citet{huang2025posthoc} on length in reward models. We claim novelty for none of it. It is here
because $\Rsens$ is uninterpretable without $\Sinv$ (Proposition~\ref{prop:degenerate}), which
makes the replication a measurement requirement rather than a contribution. What is new is the
$\Tcon$ arm, its decomposition, the human-set direction, and the profile. We note that none of
the seven bias types in \citet{unfairjudge2026} is scope or quantifier structure.

\subsection{How widespread it is: validation power of the public label sets}
\label{sec:prevalence}

Table~\ref{tab:prevalence} answers the objection raised in \S\ref{sec:intro}: if judges really
miss \strengthax\ change, why has no one reported it. It requires no annotation of ours and runs
on public artifacts at pinned revisions.

\noindent The per-item sweep returned
$\base$ between \vnv{c4AimpMin}\ and \vnv{c4AimpMax}, which reads as healthy headroom. In the
mode these sets are actually used in, a predictor reading nothing but surface form reproduces
\vnv{c4PairMin}--\vnv{c4PairMax}\ of their labels. Both columns appear in
Table~\ref{tab:prevalence} because the difference between them is the finding.

\paragraph{What the two modes measure.} Definition~\ref{def:vp} carries the mode explicitly and
Corollary~\ref{cor:modes} says why neither dominates. \textsc{RewardBench} pairs are built so the
two sides differ within a pair, but across the set no absolute length band marks the chosen side,
so a per-item predictor is near chance and a paired one is not. Our corpus is the opposite, with
its preferred side in a narrow absolute band, so both modes succeed. Reporting the per-item number
for a pairwise benchmark answers a question nobody asks of that benchmark.

\begin{table}[t]
\centering
\caption{Construct-blind recovery of the label sets the literature validates judges against, at
the revisions pinned in \texttt{results/\allowbreak labelsets/\allowbreak PROVENANCE.json}, in both input modes of
Definition~\ref{def:vp}. Per-item $\base$ is chance-corrected agreement; paired accuracy has
chance at $50\%$ with side order randomised. Group-aware folds throughout; every paired row
clears its label-permutation null, which sits at or below \vnv{c4PairNullMax}. HelpSteer2 has no
paired column because it ships per-response ratings rather than comparisons. The final row is a
provenance-inherited set audited under the identical protocol.}
\label{tab:prevalence}
\small
\input{figs/c2_prevalence}
\end{table}

\paragraph{The design assumption that fails.} \textsc{RewardBench} constrains chosen responses to
be no longer than rejected ones \citep{lambert2024rewardbench}, on the implicit assumption that
removing the naive direction of the length signal removes the length confound. Measured on the
released set, the chosen response is the longer one in $40.3\%$ of pairs, so the content-free rule
\emph{prefer the shorter response} is correct on $59.7\%$. \textbf{The confound was not removed;
its sign was flipped and its magnitude preserved}, and a reward model that learned to prefer
brevity scores above chance on the benchmark built to detect it. A control defined on a property of
the items is not a control on what a predictor can do with that property.

\paragraph{The one set that holds up, and the one that does not.}
\vnv{c4PairCleanest}\ is closest to chance in paired mode at \vnv{c4PairCleanestAcc}. It is built
by crossing every pair with three deliberate style levels, putting style variation inside the
benchmark rather than leaving it outside as a confound; this row is what makes the others
interpretable, and the design is the one we would ask others to copy.
\textsc{MT-Bench} is the opposite. Its labels are human pairwise votes, and a surface-only paired
predictor reproduces \vnv{c4PairSurfMtbench}\ of them, the highest of the public sets and higher
than either pair-role benchmark. Definition~\ref{def:pil} predicts the reverse ordering, since
labels further from the generative condition should be harder to reach without the construct.
Either the definition is incomplete or human preference on this task is itself substantially a
length preference; we cannot separate the readings here and do not report the ordering as support
for the definition (\S\ref{sec:limitations}).

\paragraph{The contrast that sets the scale.} Run the identical family, folds and nulls against a
set whose labels \emph{are} a function of the generative condition, and the paired predictor is
correct on \vnv{c4PairSurfEcp}\ of pairs against \vnv{c4PairMin}--\vnv{c4PairMax}\ for the public
sets \citep{tridehallecp2026}. So Definition~\ref{def:pil} does describe a real and severe
failure mode. What has changed is that the public benchmarks are no longer the clean comparison
they appeared to be in the per-item sweep: they sit well above chance, not at it.

\paragraph{What is still untested, stated plainly because neither mode settles it.} C4's
prediction was \emph{per axis}: that \strengthax-axis validation power would be near zero even
where aggregate power is healthy. \textbf{That prediction is untested in either mode and cannot be
tested on these sets}, because splitting $\vpow$ by axis requires the pairs in a row set to differ
by an identifiable \scopeax\ or \strengthax\ edit and none of the five is labelled that way. The
released intervention set is, so the per-axis form is answerable on it and we leave it to future
work.

\paragraph{The reversal.} Under per-item scoring the human-labelled sets look cleanest, which is
what Definition~\ref{def:pil} predicts; under paired scoring, the mode judges are actually used in,
MT-Bench moves from the cleanest set to the most recoverable. That reversal is
Corollary~\ref{cor:modes} with data attached: the two modes are not orderable, and reporting
whichever is lower is not the conservative choice it appears to be.

\paragraph{Transfer.} Within-set recovery is necessary but not sufficient
\citep{arasteh2026evidence}: a predictor that recovers a set may be exploiting a
set-specific idiom rather than a general surface correlate. We therefore fit each impoverished
predictor on one set and evaluate it on the others.

Of \vnv{xferPairs} ordered set
pairs, \vnv{xferDefined} admit the question at all: the remaining \vnv{xferUndefined} have label
spaces that do not share two categories, so a predictor fitted on one cannot be scored on the
other even in principle. \textbf{That is a finding about the state of public preference labels
rather than about our predictors}, since these sets are routinely discussed as though they
measured the same thing. On the \vnv{xferDefined} pairs where transfer is defined, the surface predictor
carries $\kappa=\vnv{xferMax}$ and $\vnv{xferMin}$, far below its within-set recovery
(Table~\ref{tab:prevalence}). The recovery is therefore substantially a set-specific idiom, which
is what makes within-set recovery necessary but not sufficient.

\subsection{What it costs: a blinded replacement, and $\Delta$}
\label{sec:consequence}

The construct is claim specificity: whether a sentence asserts more than its evidence licenses.
We choose it because a provenance-inherited label set for it already exists and is known to leak
\citep{chen2026tridehall}, so the comparison is against a real, deployed ruler rather than a
straw one.

\begin{figure}[t]
\centering
\includegraphics[width=\textwidth]{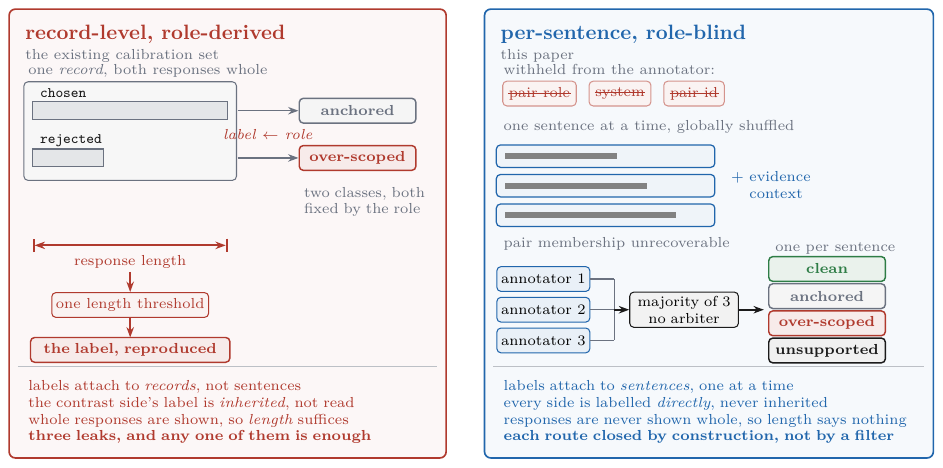}
\caption{Why a provenance-inherited label set is separable without reasoning about the
construct, and what the blinded design withholds. Left: labels attach to whole records by their
role in the pair, so one threshold on response length reproduces them, at the agreement reported
by \citet{chen2026tridehall}. Right: the unit is one sentence with its evidence context, every
field that identifies its side withheld, sides interleaved so pair membership is unrecoverable,
and the label the majority of three independent annotators with no arbitration pass. Both panels
read in two bands: above the rule, what the annotator sees and how its label is formed; below it,
the surface route the design leaves open or closes. A class chip is drawn identically wherever it
appears, so the two classes a pair role can supply compare directly against the four.}
\label{fig:blinding}
\end{figure}

\paragraph{Design.} Four per-sentence classes (\Clean, \Anchored, \Scoped, \Unsup) rather
than the two a pair role can supply. Every annotation unit is one sentence with its evidence
context, presented with no indication of which side of a pair it came from, no indication of the
producing system, sentences from both sides interleaved and globally shuffled so pair membership
is unrecoverable, and the length cue broken by presenting sentences individually rather than
whole responses. Three independent annotators, no arbitration pass and no annotator seeing another's labels,
the label taken as the majority of the three, with Fleiss' $\kappa$ and raw pairwise agreement
reported per class \citep{cohen1960kappa,fleiss1971kappa}. Figure~\ref{fig:blinding} contrasts the two designs. The full protocol,
codebook, and sampling frame are in Appendix~\ref{app:protocol}; the frozen specification is
\texttt{docs/\allowbreak TMLR\_\allowbreak EXPERIMENT\_\allowbreak PLAN.md} C3.

\paragraph{The headline would be the difference.} A per-class precision/recall table against the
blinded set alone would be a fact about one corpus's labels. $\Delta$ (how much validating on
the inherited set overstates accuracy) is the quantity that transfers, and it is the number
this design exists to produce.

\deferred{$\Delta$ itself, and $\vpow$ of both sets on identical terms. What blocks it is neither compute nor access but a second human annotation round the size of the one reported above, over a four-class codebook. The design, the protocol and the sampling frame are given in full (Appendix~\ref{app:protocol}) so that a reader can run it against their own label set: the design is the part that transfers.}

\section{Discussion and limitations}
\label{sec:discussion}

\subsection{What a judge author can do on Monday}
\label{sec:checklist}

\begin{enumerate}[leftmargin=1.4em,itemsep=2pt]
\item State how each label was obtained. If it is a function of a generative condition, say so.
\item Report the best impoverished predictor's agreement, not only the judge's.
\item Report a human ceiling with an interval, and report validation power against it.
\item Report a label-permutation null.
\item Probe in both directions: construct edits and surface edits.
\item If the label set is released, release the provenance field, so others can run 1--5.
\end{enumerate}

\noindent

\subsection{Limitations}
\label{sec:limitations}

\textbf{$\Rsens$ is a lower bound and $\Sinv$ an upper bound.} Everything rests on a $\Tcon$ edit
really changing the correct verdict; Protocol~\ref{prot:direction} puts that decision with humans,
and an item a majority of the \vnv{annRaters}\ misread enters the arm as a false positive that
lowers measured $\Rsens$. A judge may also register that a claim changed and still return the same
verdict because the verdict space is too coarse to express it. In the other direction, $\Sinv$ over
\vnv{annControlPassedSlots}\ surviving control types bounds invariance from above, since a richer
control family can only find more failures. Both bounds point the same way as the headline, which
is why \S\ref{sec:res-axes} reads \vnv{gapObs}\ as conservative.

\textbf{A third construct axis is visible at the edge of the control arm.} \vnv{annSlotsBelowBar}\
of the \vnv{nControlSlots}\ specified control types fell below the agreement bar, and the reason is
substantive rather than noise: one competent reader in three holds that making a referent more
precise changes what a sentence claims of it. If that reader is right, \emph{adding detail} is a
third axis alongside extension and precision, and it is a well-posed target for the next study. The
disagreement is localised, since the remaining control types agree at
$\ge\vnv{annControlPassedMin}$, so it is elaboration specifically and not surface-preservation in
general that is contested.

\textbf{The decomposition is a modelling choice, and not the only one the data supports.} It
accounts for $\eta^2=\vnv{audEtaReg}$ of the variance in slot-level $\Rsens$; a post-hoc regrouping
by whether an edit changes an explicit scope element or a modality-and-degree marker accounts for
\vnv{audEtaPost}. We keep the registered split and disclose the alternative, since distinguishing
them needs a construct designed to separate the two.

\textbf{The frontier is elicited, and the elicitation is part of the instrument.} Cutting a judge's
own grade makes every row of Table~\ref{tab:judges} comparable on identical terms, and it also makes
the frontier a property of the judge \emph{under our elicitation}: a judge clustering its grades on
two values yields fewer distinguishable operating points than a finer readout would show, so
Table~\ref{tab:judges} carries the realised operating-point count and regrade agreement beside each
profile.

\textbf{The diagnostic's family and mode are choices.} $\vpow$ is defined against a fixed predictor
family, so every value is an \emph{upper} bound on soundness: a richer family could only raise
$\base$. Definition~\ref{def:vp} covers two modes because those are the two we needed, and
Corollary~\ref{cor:modes} says a third mode's $\base$ is not deducible from ours. One prediction of
Definition~\ref{def:pil} does not hold on our own data: labels further from the generative
condition should be harder to reach without the construct, yet \textsc{MT-Bench}'s human votes are
the most recoverable set in paired mode. Either the definition is incomplete or human preference on
open-ended answers is itself substantially a length preference, and we cannot separate the two here.

\textbf{Scope, and provenance.} The profile is measured on the domains of \S\ref{sec:probes} and
generalises no further; constructs whose minimal edits are not well posed lie outside what this
method can measure. The paradox result is correlational, making the decomposition a
\emph{sufficient} explanation of \citet{peters2025generalization}'s finding rather than a mechanism.
The resource whose label set we audit was constructed by an overlapping author set, so we lead with
a result about judges and use no field absent from the public release, in particular no internally
retained pair-type label. Every number is against judge versions pinned in
Appendix~\ref{app:judges}.

\paragraph{Artifact availability.}
The diagnostic harness, the blinded label set, the edit taxonomy with its acceptance gates, and
the frozen judge versions are released. The corpus that supplies probe items and C4's case study is public
\citep{tridehallecp2026}; it is material here, cited, and this paper would stand with it
replaced by any resource carrying per-sentence scope annotations.

\section{Conclusion}
\label{sec:conclusion}

\begin{claimbox}
\noindent\textbf{The shortest true summary.} An evaluator is a measurement instrument, and validity
needs two properties rather than one: a verdict must hold still under an edit that preserves the
construct, \emph{and} move under a minimal edit that changes it. The judge literature has measured
the first thoroughly and the second, to our knowledge, not at all. Measured as a profile over
\vnv{profJudges}\ judges and \vnv{profDomains}\ domains with direction set by \vnv{annRaters}\
humans at agreement \vnv{annRaw}, judges sit at $\Sinv = \vnv{profSinvMean}$ against
$\Rsens = \vnv{profRsensMean}$: \emph{reliable, and markedly less responsive to construct change
than their reliability suggests}. The failure is structured rather than random. Sensitivity differs
by \vnv{profGapMean}\ between \scopeax\ and \strengthax\ at matched $\Sinv$, in that direction on
every one of the \vnv{profGapN}\ judges and against the only length bias we can measure, so the gap
is a lower bound; and the two axes respond with \emph{opposite} sign to accuracy pressure, which is
why prompting a generator to be accurate can make its overgeneralisation worse
\citep{peters2025generalization} without anything being incoherent. The reason none of this is in
the record is that the rulers leak: across \vnv{c4SetsAudited}\ public label sets, a predictor
reading only surface form, given the same paired input the judge gets, reproduces
\vnv{c4PairMin}--\vnv{c4PairMax}\ of their labels, including \vnv{c4PairSurfMtbench}\ of
\textsc{MT-Bench}'s \emph{human} votes. \emph{Whether a set leaks is a property not of the set alone
but of the set and the input mode together}, which is why Definition~\ref{def:vp} carries the mode.
\end{claimbox}

Three statements are worth carrying away separately from the result, because each is a claim about
practice. \emph{A high agreement number is not evidence of valid evaluation}; it is evidence of
invariance, which Proposition~\ref{prop:degenerate} shows a constant function attains perfectly.
\emph{An intervention that raises agreement has no predicted effect on sensitivity};
Proposition~\ref{prop:indep} says the coordinates are independent, and \S\ref{sec:res-profile} runs
ensembling and self-consistency as judge rows against exactly that prediction. \emph{And a
threshold comparison is not a validity comparison}: Proposition~\ref{prop:frontier} shows a judge
dominated everywhere on its frontier can still win at default thresholds, which is why every row of
Table~\ref{tab:judges} is read off a frontier we cut ourselves rather than off whatever operating
point a vendor shipped.

What we ask for is small, and it is not a new benchmark. Report $(\Sinv, \Rsens)$ where a judge
paper today reports agreement; report the best impoverished predictor's agreement beside the
judge's; and say how each label in the validation set was obtained, since
Definition~\ref{def:pil} is a fact about provenance that costs one sentence to disclose and cannot
be recovered afterwards. \S\ref{sec:checklist} is the six-line version.

The open question we would most like answered is whether \strengthax\ blindness is a property of
this construct or of judges in general. That needs a second construct with well-posed minimal
edits and cheap human verification, and the method transfers unchanged: the interventions, the
direction protocol and the frozen predictor family are released for it. The per-axis form of the
validation-power claim is a second piece of future work: it cannot be asked of any public row set,
since none is labelled by edit axis, and the intervention set we release is.

\bibliographystyle{plainnat}
\bibliography{references}

\appendix

\noindent The appendices carry what the argument does not need but a replication does: the full
intervention specification with its prompts, the annotation instruments, the pinned judge
configurations, the predictor family, the register of predicted values, and the secondary
experiments that inform the design without supporting a headline claim.

\section{Extended definitions and proofs}
\label{app:proofs}

\subsection{Measurability and the item distribution}

Definition~\ref{def:profile} writes $\Sinv$ and $\Rsens$ as probabilities over
$x\sim\mathcal{D}$ and $T$ drawn from an intervention family. Two details the main text
suppresses. First, the two probabilities are over \emph{different} product spaces
($\mathcal{D}\times\Tinv$ and $\mathcal{D}\times\Tcon$), so they are not two coordinates of one
joint distribution and no covariance between them is defined; this is what makes
Proposition~\ref{prop:indep} available. Second, $\mathcal{D}$ in practice is the empirical
distribution over base items that survived Protocol~\ref{prot:direction}, which is \emph{not}
the distribution over items a judge meets in deployment. Every profile in this paper is
therefore conditional on the surviving item set, and Appendix~\ref{app:secondary} reports how
much the profile moves when the survival filter is relaxed.

\subsection{The paired estimator, and why the bootstrap resamples base items}

For base items $x_1,\dots,x_n$ with $m_i$ interventions applied to $x_i$, the natural estimator
of $\Rsens$ is
\begin{equation}
\widehat{\Rsens} \;=\; \frac{1}{n}\sum_{i=1}^{n}\frac{1}{m_i}
  \sum_{j=1}^{m_i}\mathbf{1}\!\big[\Jud(x_i)\neq\Jud(T_{ij}x_i)\big],
\label{eq:est}
\end{equation}
which weights base items equally rather than interventions equally. The alternative (pooling
all $\sum_i m_i$ interventions) over-weights base items that happened to admit more edits,
and admissibility is not random: a sentence with several hedges admits several
\strengthax\ edits, so pooling would let hedge-dense sentences dominate the axis they are most
relevant to. Confidence intervals resample base items with replacement, never interventions
\citep{efron1979bootstrap}: interventions on one base item share that item's content and are not
independent draws. Nulls are label-permutation nulls over the same grouping
\citep{good2000permutation}, and where a table reports many rows we control the false discovery
rate across them \citep{benjamini1995fdr}.

\begin{figure}[t]
\centering
\includegraphics[width=\textwidth]{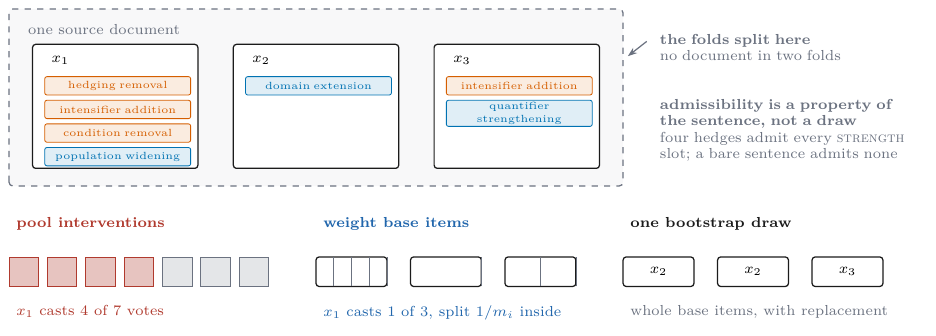}
\caption{The unit of analysis. \eqref{eq:est} weights base items equally and, within an item,
its interventions equally; pooling all interventions instead lets a hedge-dense sentence dominate
the axis it is most relevant to, since a sentence carrying several hedges admits every
\strengthax\ slot and a bare one admits none. Intervals resample whole base items with
replacement and never interventions, and folds split on the source document. The per-item counts
drawn here are illustrative; measured per-slot yields are in Table~\ref{tab:profile}.}
\label{fig:resampling}
\end{figure}

\subsection{Agreement statistics}

Judge-versus-human agreement is chance-corrected \citep{cohen1960kappa}; where the verdict space
is ordered we also report the linearly weighted variant \citep{cohen1968weighted}, because an
ordered space punishes an adjacent-class error and an opposite-class error identically under the
unweighted statistic, and the two are not equally bad for a scope judgement. Degenerate cases (both raters constant, so expected agreement is $1$ and $\kappa$ is $0/0$) are reported as
$0$ rather than as undefined, matching the harness implementation.

\section{The intervention families, in full}
\label{app:interventions}

\begin{table}[h]
\centering
\caption{Every intervention type, its class, and the axis it moves. Frozen in
\texttt{probes/\allowbreak edit\_\allowbreak taxonomy.py} before any judge was run; the assertion in that file fails if
a type is added without an axis, so the table and the code cannot drift.}
\label{tab:taxonomy}
\small
\input{figs/tab_taxonomy}
\end{table}

\paragraph{Why the control arm covers exactly these \vnv{nControlSlots}.} Three are register
manipulations. The other two exist because the judge literature has already shown verbosity
\citep{zheng2023judging} and formatting \citep{zhang2025formatbias} to move verdicts: a control
arm that omitted them would report a flattering $\Sinv$ on properties nobody had implicated, which
is the failure mode Proposition~\ref{prop:degenerate} warns about in its milder form. A sixth slot,
hedge padding, was specified and then removed: adding a hedge moves commitment, so it is a
\strengthax\ edit and a control arm containing it is not construct-preserving. The taxonomy was
internally contradictory and still looked fine as a list of names, which is why
Table~\ref{tab:taxonomy} is generated from the code rather than maintained beside it.

\paragraph{Generation and verification prompts.} The three templates below are reproduced
verbatim from \texttt{probes/\allowbreak edit\_\allowbreak taxonomy.py}, which is the single source of truth; braces
in the JSON schema are doubled there because the templates are Python format strings, and are
shown singly here. \texttt{\{edit\_type\}} is filled from Table~\ref{tab:taxonomy} and
\texttt{\{edit\_gloss\}} from the gloss column, so a template is never edited to accommodate a
slot. \texttt{\{tol\}} is the length tolerance, fixed at $10$ percent before generation.

\begin{promptbox}{Construct edit (treatment arm, $\Tcon$)}
You rewrite scientific sentences to make them claim MORE than their evidence supports, changing
nothing else.\medskip

Input sentence (well-calibrated: it states exactly what the evidence licenses): \verb|{sentence}|\\
Evidence context from the source article: \verb|{evidence}|\medskip

Apply EXACTLY ONE edit of type \verb|{edit_type}| (\verb|{edit_gloss}|).\medskip

Hard constraints. Violating any one of them makes the output useless:
\begin{enumerate}\itemsep0pt
\item Do NOT add, remove, or alter any factual content: no numbers changed, no entities swapped,
      no new findings introduced.
\item Do NOT introduce provenance language. Never write ``this study'', ``the authors'',
      ``Figure N'', ``in our data'', or any equivalent. Neither sentence may point at a source.
\item Keep the register identical. If the input is plain declarative prose, so is the output.
\item Keep the length within \verb|{tol}| percent of the input's token count. If your edit would
      shorten the sentence, compensate elsewhere WITHOUT adding content.
\item Change exactly one scope slot (the one named above) and leave the others alone.
\end{enumerate}
Return JSON only: \verb|{"edited": ..., "slot_changed": "{edit_type}", "what_widened": ...,|
\verb|"what_it_became": ..., "facts_preserved": true, "anchor_free": true}|
\end{promptbox}

\begin{promptbox}{Register edit (control arm, $\Tinv$)}
You rewrite scientific sentences to change their SURFACE FORM ONLY, leaving the claim and its
scope exactly as they were.\medskip

Input sentence: \verb|{sentence}|\medskip

Apply EXACTLY ONE edit of type \verb|{edit_type}| (\verb|{edit_gloss}|).\medskip

Hard constraints:
\begin{enumerate}\itemsep0pt
\item The claim must remain true of exactly the same population, conditions, and modality. Do
      not widen or narrow anything.
\item Do NOT add or remove factual content.
\item Do NOT introduce provenance language.
\item Keep the length within \verb|{tol}| percent of the input's token count.
\end{enumerate}
Return JSON only: \verb|{"edited": ..., "surface_change": "{edit_type}",|
\verb|"scope_unchanged": true, "facts_preserved": true, "anchor_free": true}|
\end{promptbox}

\begin{promptbox}{Verifier (family B), five gates}
Two sentences are given. Judge ONLY the following, independently of style or fluency.\medskip

A: \verb|{original}| \quad B: \verb|{edited}| \quad Evidence context: \verb|{evidence}|\medskip

Answer each question separately:
\begin{enumerate}\itemsep0pt
\item \texttt{facts\_identical}: do A and B assert the same factual content (same numbers,
      same entities, same findings) with nothing added or removed?
\item \texttt{scope\_relation}: is B's scope \texttt{broader}, \texttt{narrower}, or
      \texttt{same} relative to A?
\item \texttt{slot\_changed}: which scope slot differs, if any? One of
      \texttt{population\_\allowbreak widening}, \texttt{condition\_\allowbreak removal},
      \texttt{tense\_\allowbreak modality\_\allowbreak shift}, \texttt{quantifier\_\allowbreak strengthening},
      \texttt{hedging\_removal}, \texttt{domain\_extension}, or \texttt{none}.
\item \texttt{anchor\_\allowbreak free\_\allowbreak both}: is neither sentence pointing at a source (no ``this study'',
      no ``Figure N'', no ``the authors'')?
\item \texttt{register\_\allowbreak comparable}: are they in the same register, so that neither reads as
      more formal or more hedged than the other beyond the slot change?
\end{enumerate}
Return JSON only: \verb|{"facts_identical": <bool>, "scope_relation": ...,|
\verb|"slot_changed": ..., "anchor_free_both": <bool>, "register_comparable": <bool>,|
\verb|"notes": ...}|
\end{promptbox}

\paragraph{What the verifier gates, and what it does not.} A generated pair enters the
sensitivity arm only if the verifier returns \texttt{facts\_identical}, \texttt{anchor\_\allowbreak free\_\allowbreak both},
\texttt{register\_\allowbreak comparable}, and a \texttt{slot\_changed} equal to the requested slot; the
control arm additionally requires \texttt{scope\_relation} $=$ \texttt{same}. The verifier does
\emph{not} decide which member of a pair is the over-scoped one. That is the direction
assignment, it is made by human annotators (Appendix~\ref{app:protocol}), and it is separated
from generation and verification on purpose: a direction supplied by any model would make
$\Rsens$ a measurement of agreement between two models rather than of a judge against a
construct. Generation, verification, and judging draw on three disjoint model
families, so no family scores its own output.

\section{Annotation protocol and codebook}
\label{app:protocol}
This paper's human labels do one job: given a pair $(x, \Tcon x)$, the annotator decides which
member a correctly calibrated reader should prefer. Delegating that to a model would make $\Rsens$
an agreement statistic between two models, on which a judge sharing the direction-assigner's blind
spot scores as \emph{sensitive} exactly where both are blind.

\paragraph{Unit and blinding.} The unit is one pair, shown as sentences \textbf{A} and \textbf{B}
in randomised order with the evidence context and nothing else. Withheld: which member is the
original, which edit slot was requested, which axis it belongs to, the generating and verifying
model identities, and any batch-level composition cue. Presentation order is recorded, so a
systematic first-position preference is detectable rather than absorbed into the labels.
Treatment and control pairs are interleaved in one stream, since an annotator able to tell the
arms apart could infer the control arm's answer from the arm itself.

\begin{figure}[t]
\centering
\includegraphics[width=\textwidth]{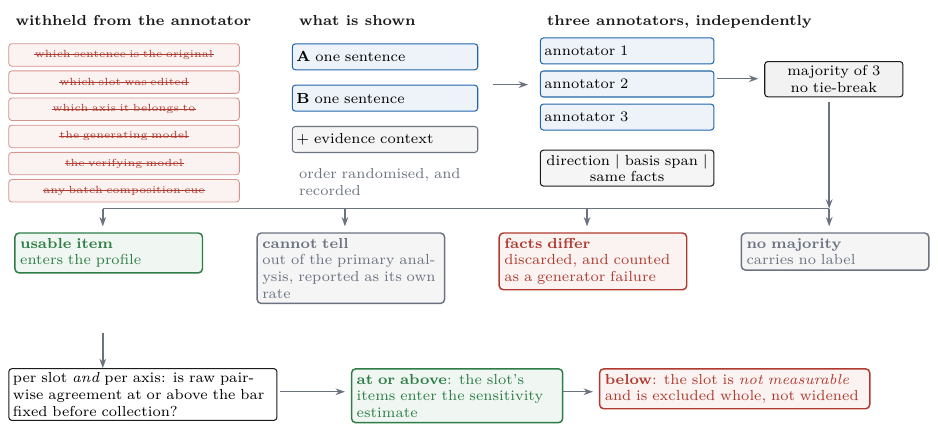}
\caption{One pair's route through Protocol~\ref{prot:direction}, and the five places it can stop.
Four outcomes are decided per item. The fifth is decided per slot \emph{and} per axis after
aggregation, against a bar fixed in \texttt{probes/\allowbreak edit\_\allowbreak taxonomy.py} before collection, and it
excludes the slot whole rather than widening an interval. Of \vnv{annItems}\ annotated pairs,
\vnv{annUsableConstruct}\ construct and \vnv{annUsableControl}\ control items survive,
\vnv{annNoMajority}\ carry no label and \vnv{annSlotsBelowBar}\ slots were excluded; per-slot
yields are in Table~\ref{tab:profile}. The \texttt{basis} field is a mandatory quotation of the
span that exceeds the evidence, which is what makes a disagreement inspectable rather than a
number, and presentation order is recorded so a systematic first-position preference is
detectable rather than absorbed into the labels.}
\label{fig:protocol}
\end{figure}

\paragraph{The judgement.} Three fields per pair. \texttt{direction} $\in$ \{\textbf{A} claims
more than the evidence licenses, \textbf{B} does, neither, cannot tell from this context\}.
\texttt{basis}: a mandatory one-sentence quotation of the exact span that exceeds what the
evidence supports, plus what the evidence does support. \texttt{same\_facts}: whether the two
sentences assert the same factual content, a check on the generator rather than a judgement about
scope. A pair marked \emph{cannot tell} leaves the primary analysis and is
reported as a separate rate; a pair where the annotators say the facts differ is discarded
outright and counted as a generator failure, because a probe built on it would measure
factuality.

\paragraph{Class definitions and the boundary cases that matter.} The four scope slots and three
strength slots of Table~\ref{tab:taxonomy} are defined for annotators through worked examples
rather than through the slot names, which are ours and not theirs. Three boundaries carry most of
the disagreement. \emph{Removal is expansion}: turning ``results suggest M may reduce noise'' into
``M reduces noise'' widens the claim although no word was added and nothing became false, and an
annotator reading only for added content will call such a pair a tie. \emph{Generality is not
always over-claiming}: a typicality quantifier the evidence supports, physical constants, unit
conversions and definitions are not over-scoped, and without this boundary an annotator drifts
toward marking every general-sounding sentence as the over-scoped one, inflating $\Rsens$ for
exactly the judges that share the heuristic. \emph{An appositive gloss is not a change of scope}:
rewriting ``aerosols'' as ``aerosols, the artificially designed particles'' names the same
referent more precisely without altering what is claimed of it. The third boundary was added
after collection, because the codebook as run did not answer the question and one annotator read
it the other way consistently (\S\ref{sec:direction}); the affected slots are reported as not
measurable rather than ruled on retrospectively.

\paragraph{Annotators, the majority rule, and the bar.} \vnv{annRaters} independent annotators
label every pair, none seeing another's labels, with domain competence in the pair's field and a
training round of shared items that is discussed and then discarded. The registered plan was two
annotators with a third arbitrating disagreements; all three labelled the full batch instead,
which is what makes a systematic disagreement distinguishable from a noisy one. The label is the
majority of three with no automatic tie-break, so an item on which all three differ carries no
label and the rate is reported (\vnv{annNoMajority} of \vnv{annItems}).

The pre-registered bar is a raw pairwise agreement of $\ge \vnv{humanAgreementMin}$, fixed in
\texttt{probes/\allowbreak edit\_\allowbreak taxonomy.py} before any data was collected and applied per axis and per
slot. Raw agreement rather than $\kappa$, because $\kappa$ is degenerate wherever one answer
legitimately dominates: \texttt{register\_shift} has raw agreement \vnv{annSlotRegKappaRaw} with
Fleiss' $\kappa$ \vnv{annSlotRegKappa}, since almost every label is \textsc{same}, and a
$\kappa$ bar would have discarded the cleanest slot in the arm. Generated tables flag every cell
where one category holds $\ge 90\%$ of the labels.

A slot or axis below the bar is reported as \textbf{not measurable} and excluded from $\Rsens$
and from the axis comparison rather than reported with a wider interval: the question is whether
the construct is well defined for humans at all, and a wide interval around an ill-defined
construct still asserts it exists.

\paragraph{Recruitment and consent.} Annotators are paid contributors, not volunteers or students
of the authors, compensated at or above the local hourly minimum for their jurisdiction. They are
told what the data will be used for, that their labels are released per item without their identity
attached, and that they may withdraw. No task presents personal data or content selected to be
distressing.

\paragraph{Two requirements on any span-quoting field.} Both cost us a batch of correct labels
before they were understood. \emph{Normalise typography before comparing}: source sentences carry
non-breaking hyphens, en dashes and curly quotes, an annotator who retypes produces the ASCII
forms, and an unnormalised substring test rejects a correct label over one character.
\emph{Accept more than one span}: an edit can have two loci that are contiguous nowhere, and a
single-span field cannot express it. A validation rule stricter than the property it checks
rejects correct work silently.

\section{Judge versions and decision protocols}
\label{app:judges}

Every judge is reached through a single aggregating API, which is what makes a multi-family study
tractable and the anti-circularity requirement of Protocol~\ref{prot:direction} easy to satisfy:
disjoint vendors are strings in a config rather than procurement problems.

\begin{warnbox}
\noindent\textbf{An aggregator makes ``pinned model version'' harder, not easier.} A model
\emph{identifier} at an aggregating API is not a model \emph{instance}: the same identifier can be
served by several upstream providers at different quantisations and with different sampling
implementations, and the default is to route to whichever is available and fall back silently.
A run that pins the judge version while accepting default routing has pinned a name, not an
instrument, and two rows of Table~\ref{tab:judges} could differ by provider rather than by judge.
We therefore pin the \emph{provider} as well as the model, disable fallback, and record the
resolved upstream per request; requests resolving to a different provider are discarded and
re-issued. The count of discards is reported, because it is a fact about the measurement
apparatus and hiding it would misrepresent how reproducible these numbers are.
\end{warnbox}

\noindent Quantisation is the specific hazard worth naming, because it is invisible and
directional: a more aggressively quantised serving of the same weights is a different instrument
for our purposes, and there is no reason its $(\Sinv,\Rsens)$ should match. Where an upstream
does not disclose its quantisation we say so in the table rather than assuming parity.

Every judge is run at temperature $0$ on the same frozen prompt, with the provider pinned through
the routing layer and the resolved provider recorded per call.

\paragraph{What the elicitation pilot did to Remark~\ref{rem:graded}.} The remark asserted that a
graded scale buys resolution over a binary verdict, and a pilot showed that it does not follow
automatically. The two failure modes are opposite: a judge can be perfectly repeatable and still
emit effectively one decision, and another can offer many usable cuts while reproducing its own
grade on the same item a small fraction of the time at temperature $0$. Both appear in the
measured roster, where usable operating points span \vnv{profPtsMin} to \vnv{profPtsMax} and
regrade agreement \vnv{profRegradeMin} to \vnv{profRegradeMax} (Table~\ref{tab:judges}). A study
reporting only \emph{distinct grades} would have called every one of them fine. The consequence is
that the score-to-verdict rule is fixed once for all judges and the regrade agreement is reported
beside every profile rather than folded into its interval.

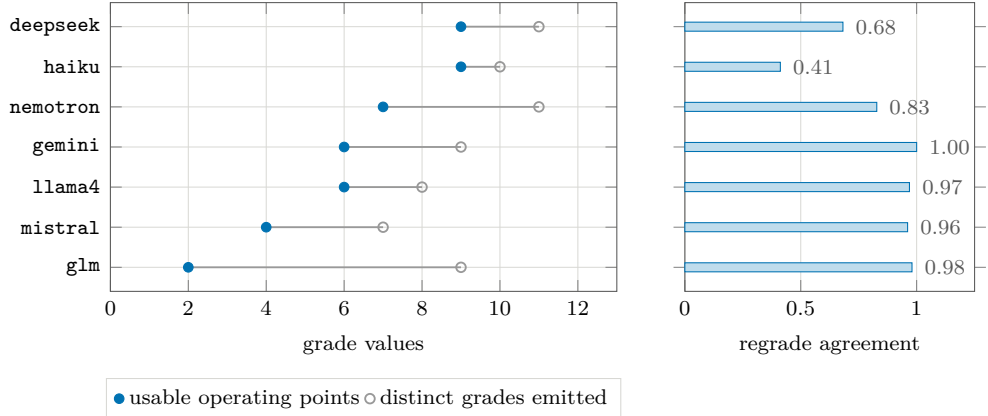
\begin{figure}[t]
\centering
\input{figs/plot_instrument}
\caption{Instrument quality per judge, and the two independent ways a graded judge fails. Left:
the grade values a judge emits, against those surviving the one-percent usability cut, so the
segment is resolution the judge appears to offer and does not deliver. Right: whether it
reproduces its own grade on the same item at temperature $0$. Rows are sorted by usable points
and the regrade column does not follow that order, so neither quantity predicts the other. Every
judge here emits at least seven distinct grades, which is why a study reporting distinct grades
would have called all of them fine.}
\label{fig:instrument}
\end{figure}

\section{The impoverished predictor family, in full}
\label{app:impoverished}

The family is deliberately weak and deliberately frozen: weak so that every reported $\vpow$ is
an upper bound and therefore generous to the label set, frozen so that $\vpow$ does not change
meaning between papers. The harness asserts the member set, and a test fails if a predictor is
added.

\begin{table}[h]
\centering
\caption{The three members. Feature counts and fitting procedure are fixed in
\texttt{analysis/\allowbreak validation\_\allowbreak power.py}; group-aware cross-validation splits on source document
so no document appears in both folds.}
\label{tab:impoverished}
\small
\begin{tabular}{@{}l p{58mm} p{46mm}@{}}
\toprule
predictor & sees & does not see \\
\midrule
\texttt{length\_only}   & response length in characters & any token identity \\
\texttt{numeric\_surface} & counts of digits, punctuation, and casing patterns & any word \\
\texttt{surface\_ngram} & character $n$-grams, $n\le 4$ & sentence structure, evidence context \\
\bottomrule
\end{tabular}
\end{table}

\paragraph{Why these three.} The members are ordered by how much of the surface they see, so a
label set's failure locates itself. If \texttt{length\_only} suffices, the set is separable by a
quantity nobody would defend as the construct. If \texttt{numeric\_surface} is needed, the leak is
in formatting or numeric density, the signature of a set whose two conditions differed in how they
render evidence. If only \texttt{surface\_ngram} succeeds, the leak is lexical: a generative
condition left an idiom behind. The three answers imply different repairs, which is why the
per-member breakdown is reported rather than only $\base = \max$.

Every member is a function of the response text alone, and none is shown the evidence context, so
none can compute the relation between claim and evidence. That is what makes $\base$
interpretable: a member reaching the human ceiling has not solved the task, it has shown the task
was not being asked. The invariant is structural, since the extractors take a single string and no
code path delivers the evidence to them.

\paragraph{Fitting.} Cross-validation is group-aware on the source document, because two items
from one document share topic, register and often exact phrasing, so a random split lets a
predictor recognise the document and recover the label through it. Each member is refit on labels
permuted within group under identical folds; refitting rather than assuming $\tfrac12$ also
catches a broken harness, since a permutation null above chance means the pipeline leaks. No
hyperparameter search is performed: a tuned impoverished predictor would make $\base$ a function
of how hard we tried. Every member ends in \texttt{LogisticRegression(max\_iter=2000,
class\_weight=\allowbreak"balanced")}, weighted because several audited sets are imbalanced and an
unweighted classifier can reach high accuracy at $\kappa$ near zero.

\begin{description}[leftmargin=1.4em,itemsep=3pt]
\item[\texttt{length\_only}.] One feature, the whitespace token count, standardised. It isolates
      the single surface property the preference literature repeatedly finds correlated with human
      choice \citep{zheng2023judging}.
\item[\texttt{numeric\_surface}.] Ten standardised scalars: character count; whitespace token
      count; count of \texttt{.,;:}; count of \texttt{(}; newline count; hedge-marker rate;
      universal-quantifier rate; digit rate; uppercase rate; mean token length. The two rate
      features use closed hand-written lists, hedges \{\emph{may, might, could, appears, suggests,
      possibly, likely, generally, typically, often, tends to, in some cases}\} and universals
      \{\emph{all, always, never, every, any, invariably, universally}\}, counted
      case-insensitively as substrings over the token count. The lists are closed so that this
      member cannot drift into a content classifier as vocabulary grows.
\item[\texttt{surface\_ngram}.] Character $n$-grams, word-boundary aware, $n\in[3,5]$, minimum
      document frequency 3, at most 20{,}000 features, sublinear term frequency, no scaler. Character $n$-grams carry some content, which is
      why this member is reported separately: recovering a set that \texttt{length\_only} does not
      means separable by \emph{style}, a weaker and different claim.
\end{description}

\paragraph{Folds and scoring.} Five folds, \texttt{GroupKFold} on the group key where the set has
at least five distinct groups and \texttt{KFold(shuffle=True)} with a fixed seed otherwise; the
output records which was used, since a $\base$ computed without grouping is not comparable to one
computed with it. A predictor's score is the $\kappa$ of the \emph{pooled} out-of-fold
predictions, not the mean of per-fold $\kappa$s: $\kappa$ is not linear in the confusion matrix,
and the pooled quantity is what Definition~\ref{def:vp} is written against. $\base$ is
chance-corrected, $\ceil$ carries its own interval, and where the two overlap we report $\vpow$ as
\emph{not resolvable}. Each member is also fit on one set and evaluated on the others, because
within-set recovery can be a set-specific idiom rather than a general surface correlate
\citep{arasteh2026evidence}.

\section{Register of predicted values}
\label{app:predicted}

Every value in this paper marked \predkey\ is listed here with the reasoning that produced it.
This appendix exists so that the predictions are falsifiable as a set rather than adjustable one
at a time after the fact: a prediction recorded before measurement is a commitment, and a
prediction recorded afterwards is a description.

Table~\ref{tab:predreg} is generated from \texttt{docs/\allowbreak PREDICTED\_\allowbreak VALUES.md} by
\texttt{analysis/\allowbreak make\_\allowbreak predicted\_\allowbreak register.py}, which also fails the build if the paper renders
a \predkey\ value the register does not explain. Transcribing it by hand would defeat its
purpose: the appendix is only a commitment device if it cannot drift from the file the
commitments were written in.

Two properties of the register are worth stating because they constrain how the rest of the
paper may argue. First, \textbf{no sentence in this paper argues \emph{from} a predicted
value.} Where a section's design depends on one, the dependence is named in the prose. Second,
predictions exist here for a mechanical reason rather than a rhetorical one: a plot needs
coordinates to render, and there is no hole-marker equivalent for a data point, so the choice for a
figure is between a marked prediction and no figure at all. Given that choice, the useful move
is to record the prediction with its refutation condition attached and let the experiment kill
it. Every such figure carries a \textsc{predicted} watermark drawn inside the axis rather than
in the caption, because a caption disclaimer does not survive the figure being screenshotted
into a slide.

This appendix is a defect that shrinks. Each line disappears when its measurement lands, and
the appendix disappears with the last one.

\input{figs/tab_predicted_register}

\section{What is not measured here}
\label{app:notmeasured}

\input{figs/tab_notmeasured}

\section{Secondary experiments and background work}
\label{app:secondary}

None of the following supports a headline claim. Each rules out an alternative explanation, bounds
a design choice, or records a path considered and rejected.

\paragraph{S1, S3, S6: robustness checks specified and not run.} Three checks are named in
Table~\ref{tab:notmeasured} with the reason each awaits. \emph{Reliability remedies} (S1):
Proposition~\ref{prop:indep} predicts that self-consistency voting and a judge ensemble raise
$\Sinv$ without moving $\Rsens$; both coordinates are reported for every remedy row, never
$\Sinv$ alone. \emph{Edit magnitude} (S2) has run and is reported in
\S\ref{sec:res-axes}; the one thing it cannot do is stratify, because operation fixes magnitude. \emph{Verdict granularity} (S3) asks how much insensitivity survives a
finer verdict space. The elicitation pilot already answers half of it, and not in the direction the
check assumed: a finer space does not automatically buy resolution, since one judge yields
\vnv{profPtsMin} usable operating points and another reproduces its own grade on the same item
\vnv{profRegradeMin} of the time at temperature $0$. ``Repeat with a graded score'' is therefore
not a strictly more informative measurement. \emph{Item-survival sensitivity} (S6) recomputes the
profile at relaxed agreement bars; if the asymmetry grows as the bar relaxes it is partly an
artefact of which items survive.

\paragraph{S4. A free replication of the asymmetry, from the pair-construction pipeline.} Building
the intervention set requires a verifier model to check that each edit preserved the facts, the
anchor-free property and the register (Appendix~\ref{app:interventions}). That verifier is also
asked, for its own record, whether the edited sentence claims more than the original: the same
question put to the annotators, put to a model.

It notices \vnv{vfDetectScope} of \scopeax\ edits (\vnv{vfNScope} pairs) and
\vnv{vfDetectStrength} of \strengthax\ edits (\vnv{vfNStrength} pairs), a factor of
\vnv{vfDetectRatio}. Where it does register a change it assigns the same axis we do
\vnv{vfAxisAgree} of the time (\vnv{vfAxisAgreeN} pairs), so the gap is not an artefact of
disagreeing about what the axes mean.

\begin{warnbox}
\noindent\textbf{This is not $\Rsens$.} The verifier is not a judge, the task is not the judge's
task, and no human assigned direction to these pairs, so there is no ground truth here: only a
model's agreement with our construction. It is reported as an independent instrument, built for
another purpose, showing the predicted asymmetry at no additional cost.

The verifier's own insensitivity is why direction is a human decision. An earlier build gated the
construct arm on the verifier agreeing that an edit widened the claim, and that gate rejected 12
of 12 verified strength edits, each with a note conceding the change and denying it mattered. Such
a gate keeps only the strength edits a model can already see, so every judge would then be scored
on a set pre-filtered for detectability. The verifier's opinion is recorded rather than enforced.
\end{warnbox}

\paragraph{S5. Alternative decompositions considered and rejected.} Before settling on
\scopeax/\allowbreak\strengthax\ we considered a single graded overreach magnitude, rejected because it cannot
express the opposite-sign prediction of \S\ref{sec:paradox}; a three-way split separating
quantifier from temporal generalisation, rejected because almost every temporal generalisation in
the domain's writing conventions also strengthens a quantifier, so the two are not separable by a
minimal edit; and a split by \emph{evidence} type rather than claim operation, rejected because it
makes membership depend on the source document and so cannot be verified from the pair alone. The
chosen decomposition is a modelling choice, and a reader should know what it was chosen against.

\end{document}

%% file: preamble.tex
\usepackage[a4paper,margin=1in]{geometry}
\usepackage[T1]{fontenc}
\usepackage[utf8]{inputenc}
\usepackage{lmodern}
\usepackage{amsmath,amssymb,amsthm}
\usepackage{mathtools}
\usepackage{graphicx}
\usepackage{array}
\usepackage{booktabs}
\usepackage{longtable}
\usepackage{tabularx}
\usepackage{multirow}
\usepackage{float}
\usepackage[section]{placeins}
\usepackage{authblk}
\usepackage[dvipsnames]{xcolor}
\usepackage[most]{tcolorbox}
\usepackage{microtype}
\usepackage{enumitem}
\usepackage[numbers,sort&compress]{natbib}
\usepackage{caption}
\usepackage{abstract}
\usepackage{tikz}
\usepackage{pgfplots}
\pgfplotsset{compat=1.18}
\usetikzlibrary{arrows.meta,positioning,fit,backgrounds,patterns,calc,
  shapes.geometric,decorations.pathreplacing,decorations.markings}
\usepackage[colorlinks=true,linkcolor=NavyBlue,citecolor=ForestGreen,urlcolor=NavyBlue]{hyperref}
\usepackage{xurl}

\renewcommand{\arraystretch}{1.2}
\newcolumntype{L}{>{\raggedright\arraybackslash}X}

\newtheorem{proposition}{Proposition}

\newtheorem{corollary}{Corollary}
\theoremstyle{definition}
\newtheorem{definition}{Definition}

\newtheorem{protocol}{Protocol}
\theoremstyle{remark}
\newtheorem{remark}{Remark}

\newtcolorbox{claimbox}{breakable,colback=ForestGreen!4,colframe=ForestGreen!55,boxrule=1pt,
  arc=3pt,left=8pt,right=8pt,top=6pt,bottom=6pt}
\newtcolorbox{warnbox}{breakable,colback=BrickRed!4,colframe=BrickRed!55,boxrule=1pt,
  arc=3pt,left=8pt,right=8pt,top=6pt,bottom=6pt}
\newtcolorbox{firewallbox}{breakable,colback=Goldenrod!8,colframe=Goldenrod!70,boxrule=1pt,
  arc=3pt,left=8pt,right=8pt,top=6pt,bottom=6pt}

\newtcolorbox{promptbox}[1]{breakable,colback=black!2,colframe=black!45,boxrule=0.7pt,
  arc=2pt,left=7pt,right=7pt,top=5pt,bottom=5pt,fonttitle=\bfseries\small,
  title={#1},coltitle=white,colbacktitle=black!55,
  before upper=\small\setlength{\parskip}{2pt}}

\newcommand{\deferred}[1]{\emph{Not measured here.} #1}
\usepackage[normalem]{ulem}

\newcommand{\Clean}{\textsc{clean}}
\newcommand{\Scoped}{\textsc{over-scoped}}
\newcommand{\Anchored}{\textsc{anchored}}
\newcommand{\Unsup}{\textsc{unsupported}}
\newcommand{\vpow}{\mathrm{VP}}          
\newcommand{\ceil}{\mathrm{a}_{\mathrm{hum}}}   
\newcommand{\base}{\mathrm{a}_{\mathrm{imp}}}   
\newcommand{\Lset}{\mathcal{L}}          
\newcommand{\Cgen}{C}                    

\newcommand{\Jud}{J}                     
\newcommand{\Crit}{c}                    
\newcommand{\Tinv}{T_{\mathrm{I}}}       
\newcommand{\Tcon}{T_{\mathrm{C}}}       
\newcommand{\Sinv}{S}                    
\newcommand{\Rsens}{R}                   
\newcommand{\Vprof}{\mathbf{V}}          
\newcommand{\pnv}[1]{%
  \ifcsname pnv-#1\endcsname
    \textcolor{figPred}{\csname pnv-#1\endcsname\textsuperscript{\dag}}%
  \else\textcolor{BrickRed}{\textbf{??pred:#1}}\fi}
\newcommand{\predkey}{\textcolor{figPred}{\textsuperscript{\dag}}\,predicted, not measured}

\newcommand{\scopeax}{\textsc{scope}}
\newcommand{\strengthax}{\textsc{strength}}

%% file: palette.tex
\definecolor{figA}{HTML}{0072B2}
\definecolor{figB}{HTML}{D55E00}
\definecolor{figC}{HTML}{009E73}
\definecolor{figD}{HTML}{CC79A7}
\definecolor{figE}{HTML}{56B4E9}
\definecolor{figF}{HTML}{E69F00}
\definecolor{figG}{HTML}{999999}
\definecolor{figInk}{HTML}{1A1A1A}
\definecolor{figSoft}{HTML}{5A5A5A}
\definecolor{figGrid}{HTML}{D8D8D4}
\definecolor{figPred}{HTML}{B03A2E}

%% file: figs/vwv_numbers.tex
\newcommand{\vnv}[1]{%
  \ifcsname vnv-#1\endcsname\csname vnv-#1\endcsname
  \else\textcolor{BrickRed}{\textbf{??#1}}\fi}
\expandafter\def\csname vnv-relVerbosityBias\endcsname{0.011}
\expandafter\def\csname vnv-relJudges\endcsname{21}
\expandafter\def\csname vnv-relJudgments\endcsname{541{,}000}
\expandafter\def\csname vnv-humanAgreementMin\endcsname{0.75}
\expandafter\def\csname vnv-nEditSlots\endcsname{7}
\expandafter\def\csname vnv-nScopeSlots\endcsname{4}
\expandafter\def\csname vnv-nStrengthSlots\endcsname{3}
\expandafter\def\csname vnv-nControlSlots\endcsname{5}
\expandafter\def\csname vnv-profJudges\endcsname{7}
\expandafter\def\csname vnv-profJudgesReached\endcsname{7}
\expandafter\def\csname vnv-profSinvTarget\endcsname{0.90}
\expandafter\def\csname vnv-profPairs\endcsname{342}
\expandafter\def\csname vnv-profNCons\endcsname{246}
\expandafter\def\csname vnv-profNCtrl\endcsname{96}
\expandafter\def\csname vnv-profDomains\endcsname{4}
\expandafter\def\csname vnv-profRsensMean\endcsname{0.319}
\expandafter\def\csname vnv-profScopeMean\endcsname{0.383}
\expandafter\def\csname vnv-profStrengthMean\endcsname{0.262}
\expandafter\def\csname vnv-profSinvMean\endcsname{0.945}
\expandafter\def\csname vnv-profGapMean\endcsname{+0.121}
\expandafter\def\csname vnv-profGapMin\endcsname{+0.003}
\expandafter\def\csname vnv-profGapMax\endcsname{+0.222}
\expandafter\def\csname vnv-profGapPos\endcsname{7}
\expandafter\def\csname vnv-profGapN\endcsname{7}
\expandafter\def\csname vnv-profBestJudge\endcsname{anthropic/claude-haiku-4.5}
\expandafter\def\csname vnv-profBestRsens\endcsname{0.561}
\expandafter\def\csname vnv-gapObs\endcsname{+0.121}
\expandafter\def\csname vnv-gapLo\endcsname{+0.052}
\expandafter\def\csname vnv-gapHi\endcsname{+0.189}
\expandafter\def\csname vnv-gapBaseItems\endcsname{243}
\expandafter\def\csname vnv-gapPerm\endcsname{$p<0.002$}
\expandafter\def\csname vnv-gapNullMax\endcsname{+0.123}
\expandafter\def\csname vnv-gapUnan\endcsname{+0.125}
\expandafter\def\csname vnv-gapUnanN\endcsname{217}
\expandafter\def\csname vnv-profCtrlFPR\endcsname{0.055}
\expandafter\def\csname vnv-profPtsMin\endcsname{2}
\expandafter\def\csname vnv-profPtsMax\endcsname{9}
\expandafter\def\csname vnv-profRegradeMin\endcsname{0.412}
\expandafter\def\csname vnv-profRegradeMax\endcsname{1.000}
\expandafter\def\csname vnv-profRegradeWorstJudge\endcsname{anthropic/claude-haiku-4.5}
\expandafter\def\csname vnv-profRegradeWorstPts\endcsname{9}
\expandafter\def\csname vnv-profFlatJudges\endcsname{1}
\expandafter\def\csname vnv-profFlatJudge\endcsname{z-ai/glm-4.7-flash}
\expandafter\def\csname vnv-audMagScope\endcsname{5.0\%}
\expandafter\def\csname vnv-audMagStrength\endcsname{5.3\%}
\expandafter\def\csname vnv-audSignScope\endcsname{+4.1\%}
\expandafter\def\csname vnv-audSignStrength\endcsname{-2.5\%}
\expandafter\def\csname vnv-audLongerScope\endcsname{64\%}
\expandafter\def\csname vnv-audLongerStrength\endcsname{39\%}
\expandafter\def\csname vnv-audCorrMin\endcsname{+0.025}
\expandafter\def\csname vnv-audCorrMax\endcsname{+0.114}
\expandafter\def\csname vnv-audGapLong\endcsname{+0.149}
\expandafter\def\csname vnv-audGapLongPos\endcsname{6}
\expandafter\def\csname vnv-audGapLongN\endcsname{7}
\expandafter\def\csname vnv-audGapShort\endcsname{+0.073}
\expandafter\def\csname vnv-audGapShortPos\endcsname{6}
\expandafter\def\csname vnv-audGapShortN\endcsname{7}
\expandafter\def\csname vnv-audBestSlot\endcsname{quantifier\_strengthening}
\expandafter\def\csname vnv-audBestSlotR\endcsname{0.445}
\expandafter\def\csname vnv-audWorstSlot\endcsname{hedging\_removal}
\expandafter\def\csname vnv-audWorstSlotR\endcsname{0.210}
\expandafter\def\csname vnv-audCrossSlot\endcsname{condition\_removal}
\expandafter\def\csname vnv-audCrossSlotR\endcsname{0.379}
\expandafter\def\csname vnv-audCrossBeats\endcsname{2}
\expandafter\def\csname vnv-audEtaReg\endcsname{0.40}
\expandafter\def\csname vnv-audEtaPost\endcsname{0.50}
\expandafter\def\csname vnv-audGapPost\endcsname{+0.136}
\expandafter\def\csname vnv-annRaters\endcsname{3}
\expandafter\def\csname vnv-annItems\endcsname{478}
\expandafter\def\csname vnv-annRaw\endcsname{0.852}
\expandafter\def\csname vnv-annKappa\endcsname{0.776}
\expandafter\def\csname vnv-annPairMin\endcsname{0.814}
\expandafter\def\csname vnv-annPairMax\endcsname{0.877}
\expandafter\def\csname vnv-annScopeRaw\endcsname{0.898}
\expandafter\def\csname vnv-annScopeN\endcsname{144}
\expandafter\def\csname vnv-annStrengthRaw\endcsname{0.889}
\expandafter\def\csname vnv-annStrengthN\endcsname{144}
\expandafter\def\csname vnv-annControlRaw\endcsname{0.789}
\expandafter\def\csname vnv-annControlN\endcsname{190}
\expandafter\def\csname vnv-annNoMajority\endcsname{2}
\expandafter\def\csname vnv-annNoMajorityPct\endcsname{0.4\%}
\expandafter\def\csname vnv-annSlotsBelowBar\endcsname{2}
\expandafter\def\csname vnv-annSlotElaborationRaw\endcsname{0.404}
\expandafter\def\csname vnv-annSlotVerbosityPaddingRaw\endcsname{0.689}
\expandafter\def\csname vnv-annControlPassedMin\endcsname{0.971}
\expandafter\def\csname vnv-annControlPassedSlots\endcsname{3}
\expandafter\def\csname vnv-annSlotRegName\endcsname{register\_shift}
\expandafter\def\csname vnv-annSlotRegKappaRaw\endcsname{0.976}
\expandafter\def\csname vnv-annSlotRegKappa\endcsname{0.324}
\expandafter\def\csname vnv-annConstructAnnotated\endcsname{288}
\expandafter\def\csname vnv-annControlAnnotated\endcsname{190}
\expandafter\def\csname vnv-annUsableConstruct\endcsname{246}
\expandafter\def\csname vnv-annUsableControl\endcsname{96}
\expandafter\def\csname vnv-annControlConfirmsBeforeExcl\endcsname{150}
\expandafter\def\csname vnv-annUsableControl\endcsname{96}
\expandafter\def\csname vnv-annUsableScope\endcsname{116}
\expandafter\def\csname vnv-annUsableStrength\endcsname{130}
\expandafter\def\csname vnv-annTenseYield\endcsname{57.8\%}
\expandafter\def\csname vnv-annTenseUsable\endcsname{26}
\expandafter\def\csname vnv-annTenseN\endcsname{45}
\expandafter\def\csname vnv-annYieldMin\endcsname{57.8\%}
\expandafter\def\csname vnv-annYieldMinSlot\endcsname{tense\_modality\_shift}
\expandafter\def\csname vnv-xferPairs\endcsname{20}
\expandafter\def\csname vnv-xferDefined\endcsname{2}
\expandafter\def\csname vnv-xferUndefined\endcsname{18}
\expandafter\def\csname vnv-xferMax\endcsname{0.157}
\expandafter\def\csname vnv-xferMin\endcsname{0.109}
\expandafter\def\csname vnv-pipelineDiscard\endcsname{24.2\%}
\expandafter\def\csname vnv-pipelineKept\endcsname{5{,}539}
\expandafter\def\csname vnv-annGateFlagged\endcsname{47}
\expandafter\def\csname vnv-annGateOf\endcsname{190}
\expandafter\def\csname vnv-annGateFlagDir\endcsname{59.6\%}
\expandafter\def\csname vnv-annGateCleanDir\endcsname{7.0\%}
\expandafter\def\csname vnv-annPadPreN\endcsname{46}
\expandafter\def\csname vnv-annPadPreDir\endcsname{67.4\%}
\expandafter\def\csname vnv-annPadPreAgree\endcsname{0.601}
\expandafter\def\csname vnv-annPadPreCallsHi\endcsname{37}
\expandafter\def\csname vnv-annPadPreCallsLo\endcsname{12}
\expandafter\def\csname vnv-annPadPostN\endcsname{46}
\expandafter\def\csname vnv-annPadPostDir\endcsname{10.9\%}
\expandafter\def\csname vnv-annPadPostAgree\endcsname{0.572}
\expandafter\def\csname vnv-annPadPostCallsHi\endcsname{32}
\expandafter\def\csname vnv-annPadPostCallsLo\endcsname{3}
\expandafter\def\csname vnv-vfDetectScope\endcsname{24.1\%}
\expandafter\def\csname vnv-vfNScope\endcsname{957}
\expandafter\def\csname vnv-vfDetectStrength\endcsname{8.8\%}
\expandafter\def\csname vnv-vfNStrength\endcsname{794}
\expandafter\def\csname vnv-vfDetectRatio\endcsname{2.7}
\expandafter\def\csname vnv-vfAxisAgree\endcsname{82\%}
\expandafter\def\csname vnv-vfAxisAgreeN\endcsname{301}
\expandafter\def\csname vnv-c4SetsAudited\endcsname{5}
\expandafter\def\csname vnv-c4AimpMin\endcsname{0.07}
\expandafter\def\csname vnv-c4AimpMax\endcsname{0.30}
\expandafter\def\csname vnv-c4AimpMedian\endcsname{0.22}
\expandafter\def\csname vnv-c4ItemsTotal\endcsname{37{,}627}
\expandafter\def\csname vnv-c4AimpRewardbench\endcsname{0.30}
\expandafter\def\csname vnv-c4LenRewardbench\endcsname{0.11}
\expandafter\def\csname vnv-c4AimpRewardbench2\endcsname{0.20}
\expandafter\def\csname vnv-c4LenRewardbench2\endcsname{0.07}
\expandafter\def\csname vnv-c4AimpRmbench\endcsname{0.22}
\expandafter\def\csname vnv-c4LenRmbench\endcsname{0.02}
\expandafter\def\csname vnv-c4AimpHelpsteer2\endcsname{0.07}
\expandafter\def\csname vnv-c4LenHelpsteer2\endcsname{0.03}
\expandafter\def\csname vnv-c4AimpMtbench\endcsname{0.30}
\expandafter\def\csname vnv-c4LenMtbench\endcsname{0.08}
\expandafter\def\csname vnv-c4LengthNullSets\endcsname{RM-Bench}
\expandafter\def\csname vnv-c4CleanestSet\endcsname{HelpSteer2}
\expandafter\def\csname vnv-c4CleanestAimp\endcsname{0.07}
\expandafter\def\csname vnv-c4CleanestOrigin\endcsname{human rating}
\expandafter\def\csname vnv-c4PairLenRewardbench\endcsname{58.0\%}
\expandafter\def\csname vnv-c4PairLenKRewardbench\endcsname{0.157}
\expandafter\def\csname vnv-c4PairSurfRewardbench\endcsname{62.4\%}
\expandafter\def\csname vnv-c4PairSurfKRewardbench\endcsname{0.249}
\expandafter\def\csname vnv-c4PairLenRewardbench2\endcsname{57.4\%}
\expandafter\def\csname vnv-c4PairLenKRewardbench2\endcsname{0.149}
\expandafter\def\csname vnv-c4PairSurfRewardbench2\endcsname{62.4\%}
\expandafter\def\csname vnv-c4PairSurfKRewardbench2\endcsname{0.248}
\expandafter\def\csname vnv-c4PairLenRmbench\endcsname{54.7\%}
\expandafter\def\csname vnv-c4PairLenKRmbench\endcsname{0.095}
\expandafter\def\csname vnv-c4PairSurfRmbench\endcsname{53.1\%}
\expandafter\def\csname vnv-c4PairSurfKRmbench\endcsname{0.062}
\expandafter\def\csname vnv-c4PairLenMtbench\endcsname{66.4\%}
\expandafter\def\csname vnv-c4PairLenKMtbench\endcsname{0.328}
\expandafter\def\csname vnv-c4PairSurfMtbench\endcsname{67.4\%}
\expandafter\def\csname vnv-c4PairSurfKMtbench\endcsname{0.348}
\expandafter\def\csname vnv-c4PairLenEcp\endcsname{97.1\%}
\expandafter\def\csname vnv-c4PairLenKEcp\endcsname{0.942}
\expandafter\def\csname vnv-c4PairSurfEcp\endcsname{99.6\%}
\expandafter\def\csname vnv-c4PairSurfKEcp\endcsname{0.993}
\expandafter\def\csname vnv-c4PairMin\endcsname{55\%}
\expandafter\def\csname vnv-c4PairMax\endcsname{67\%}
\expandafter\def\csname vnv-c4PairCleanest\endcsname{RM-Bench}
\expandafter\def\csname vnv-c4PairCleanestAcc\endcsname{54.7\%}
\expandafter\def\csname vnv-c4PairNullMax\endcsname{0.026}
\expandafter\def\csname vnv-c4EcpZhAimp\endcsname{0.993}
\expandafter\def\csname vnv-c4EcpZhN\endcsname{10{,}156}
\expandafter\def\csname vnv-c4EcpZhLen\endcsname{0.842}
\expandafter\def\csname vnv-c4EcpEnAimp\endcsname{0.978}
\expandafter\def\csname vnv-c4EcpEnN\endcsname{10{,}042}
\expandafter\def\csname vnv-c4EcpEnLen\endcsname{0.901}

%% file: figs/vwv_predicted.tex
\expandafter\def\csname pnv-Sinv\endcsname{0.88}
\expandafter\def\csname pnv-Rsens\endcsname{0.41}
\expandafter\def\csname pnv-RscopeAx\endcsname{0.57}
\expandafter\def\csname pnv-RstrengthAx\endcsname{0.24}
\expandafter\def\csname pnv-axisGap\endcsname{0.33}
\expandafter\def\csname pnv-ctrlFPR\endcsname{0.12}
\expandafter\def\csname pnv-dScopeAcc\endcsname{-0.09}
\expandafter\def\csname pnv-dStrengthAcc\endcsname{+0.14}
\expandafter\def\csname pnv-vpMedian\endcsname{0.11}
\expandafter\def\csname pnv-vpStrength\endcsname{0.03}
\expandafter\def\csname pnv-Delta\endcsname{0.19}
\expandafter\def\csname pnv-nJudges\endcsname{6}
\expandafter\def\csname pnv-nDomains\endcsname{4}
\expandafter\def\csname pnv-humanAgree\endcsname{0.81}

%% file: authors.tex
\newif\ifanon
\anonfalse                    

\ifanon
  \author{}
  \date{\small\textcolor{BrickRed}{\textbf{ANONYMISED BUILD} --- author block suppressed}}
\else
  \author[1]{Jianlin Chen\thanks{\texttt{202330450231@mail.scut.edu.cn}}}
  \author[2]{Wenhui Chen\thanks{\texttt{mc35092@um.edu.mo}}}
  \author[2]{Ziyao Lin\thanks{\texttt{mc35081@um.edu.mo}}}
  \author[2]{Chi Man Vong\thanks{Corresponding author. \texttt{cmvong@um.edu.mo}}}
  \affil[1]{South China University of Technology}
  \affil[2]{University of Macau}
  \date{}
\fi

%% file: figs/tab_annotation_float.tex
\begin{table}[t]
\centering
\small
\caption{Human direction annotation, per slot. \emph{agree} is mean pairwise agreement over
three annotators; \emph{usable} counts items whose majority label matches the intervention's
intended direction. The two columns answer different questions and a slot can fail either one
alone: \texttt{tense\_modality\_shift} has ordinary agreement and low yield --- readers agreeing
that the edit changed nothing --- while the two length-manipulating invariance slots have
ordinary yield and agreement below the pre-set bar of \vnv{humanAgreementMin}.}
\label{tab:annotation}
\input{figs/tab_annotation}
\end{table}

%% file: figs/tab_annotation.tex
\begin{tabular}{@{}llrrrl@{}}
\toprule
slot & axis & $n$ & agree & usable & note \\
\midrule
\texttt{domain\_extension} & \scopeax & 25 & 0.973 & 25 &  \\
\texttt{population\_widening} & \scopeax & 42 & 0.857 & 34 &  \\
\texttt{quantifier\_strengthening} & \scopeax & 32 & 0.958 & 31 &  \\
\texttt{tense\_modality\_shift} & \scopeax & 45 & 0.852 & 26 & low yield \\
\midrule
\texttt{condition\_removal} & \strengthax & 35 & 0.867 & 33 &  \\
\texttt{hedging\_removal} & \strengthax & 58 & 0.851 & 48 &  \\
\texttt{intensifier\_addition} & \strengthax & 51 & 0.948 & 49 &  \\
\midrule
\texttt{elaboration} & invariance & 33 & 0.404 & 0 & \textbf{below bar; excluded} \\
\texttt{format\_shift} & invariance & 20 & 1.000 & 20 &  \\
\texttt{register\_shift} & invariance & 55 & 0.976 & 54 &  \\
\texttt{reordering} & invariance & 23 & 0.971 & 22 &  \\
\texttt{verbosity\_padding} & invariance & 59 & 0.689 & 0 & \textbf{below bar; excluded} \\
\bottomrule
\end{tabular}

%% file: figs/c4_profile.tex
\begin{tabular}{@{}llrrrc@{}}
\toprule
intervention & axis & $n$ & human dir. & agreement & flip rate \\
\midrule
\multicolumn{6}{@{}l}{\itshape $\Tcon$ --- should flip: \scopeax} \\
domain extension & scope & 25 & 100\% & 0.973 & 0.445 \\
population widening & scope & 34 & 81\% & 0.857 & 0.299 \\
quantifier strengthening & scope & 31 & 97\% & 0.958 & 0.445 \\
tense modality shift & scope & 26 & 58\% & 0.852 & 0.358 \\
\addlinespace
\multicolumn{6}{@{}l}{\itshape $\Tcon$ --- should flip: \strengthax} \\
condition removal & strength & 33 & 94\% & 0.867 & 0.379 \\
hedging removal & strength & 48 & 83\% & 0.851 & 0.210 \\
intensifier addition & strength & 49 & 96\% & 0.948 & 0.235 \\
\addlinespace
\multicolumn{6}{@{}l}{\itshape $\Tinv$ --- should not flip: surface controls} \\
elaboration & control & 0 & --- & \textbf{0.404}$^\dagger$ & --- \\
format shift & control & 20 & 100\% & 1.000 & 0.029 \\
register shift & control & 54 & 98\% & 0.976 & 0.053 \\
reordering & control & 22 & 96\% & 0.971 & 0.084 \\
verbosity padding & control & 0 & --- & \textbf{0.689}$^\dagger$ & --- \\
\midrule
\multicolumn{2}{@{}l}{$\Vprof=(\Sinv,\Rsens)$, mean over judges}
                          & \multicolumn{4}{l}{$(\vnv{profSinvMean},\ \vnv{profRsensMean})$} \\
\bottomrule
\end{tabular}

%% file: figs/plot_profiles.tex
\begin{tikzpicture}
\begin{axis}[
  width=0.80\textwidth, height=6.2cm,
  xlabel={$\Sinv$ (invariance)}, ylabel={$\Rsens$ (sensitivity)},
  xmin=0, xmax=1.02, ymin=0, ymax=1.0,
  grid=major, grid style={figGrid, line width=0.3pt},
  tick label style={font=\scriptsize}, label style={font=\scriptsize},
  legend style={font=\scriptsize, at={(0.02,0.02)}, anchor=south west,
                draw=figGrid, legend columns=2},
  legend cell align=left, axis line style={figSoft},
]
\addplot[figSoft, dashed, line width=0.5pt, forget plot] coordinates {(0,1)(1,0)};
\addplot[figA, line width=0.7pt, mark=*, mark size=0.9pt] coordinates {(0.000,0.886) (0.146,0.874) (0.292,0.841) (0.490,0.825) (0.615,0.801) (0.677,0.752) (0.781,0.732) (0.833,0.715) (0.875,0.687) (0.885,0.659) (0.917,0.561) (0.948,0.508) (0.958,0.435) (0.969,0.386) (0.979,0.325) (0.990,0.215) (1.000,0.057)};
\addlegendentry{\texttt{haiku}}
\addplot[figA, only marks, mark=o, mark size=2.6pt, line width=0.8pt, forget plot] coordinates {(0.917,0.561)};
\addplot[figB, line width=0.7pt, mark=*, mark size=0.9pt] coordinates {(0.000,0.667) (0.906,0.447) (0.917,0.337) (0.927,0.264) (0.948,0.236) (0.958,0.159) (0.969,0.146) (0.979,0.138) (0.990,0.106) (1.000,0.057)};
\addlegendentry{\texttt{nemotron}}
\addplot[figB, only marks, mark=o, mark size=2.6pt, line width=0.8pt, forget plot] coordinates {(0.906,0.447)};
\addplot[figC, line width=0.7pt, mark=*, mark size=0.9pt] coordinates {(0.000,0.698) (0.968,0.408) (0.979,0.261) (0.989,0.082) (1.000,0.045)};
\addlegendentry{\texttt{gemini}}
\addplot[figC, only marks, mark=o, mark size=2.6pt, line width=0.8pt, forget plot] coordinates {(0.968,0.408)};
\addplot[figD, line width=0.7pt, mark=*, mark size=0.9pt] coordinates {(0.000,0.772) (0.625,0.622) (0.667,0.602) (0.958,0.309) (0.990,0.106) (1.000,0.012)};
\addlegendentry{\texttt{mistral}}
\addplot[figD, only marks, mark=o, mark size=2.6pt, line width=0.8pt, forget plot] coordinates {(0.958,0.309)};
\addplot[figE, line width=0.7pt, mark=*, mark size=0.9pt] coordinates {(0.000,0.780) (0.719,0.714) (0.729,0.706) (0.771,0.580) (0.812,0.429) (0.854,0.367) (0.865,0.269) (0.906,0.245) (0.969,0.171) (0.990,0.110) (1.000,0.008)};
\addlegendentry{\texttt{deepseek}}
\addplot[figE, only marks, mark=o, mark size=2.6pt, line width=0.8pt, forget plot] coordinates {(0.906,0.245)};
\addplot[figF, line width=0.7pt, mark=*, mark size=0.9pt] coordinates {(0.000,0.529) (0.979,0.168) (0.989,0.084)};
\addlegendentry{\texttt{glm}}
\addplot[figF, only marks, mark=o, mark size=2.6pt, line width=0.8pt, forget plot] coordinates {(0.979,0.168)};
\addplot[figG, line width=0.7pt, mark=*, mark size=0.9pt] coordinates {(0.000,0.695) (0.667,0.492) (0.698,0.488) (0.740,0.451) (0.979,0.093) (0.990,0.037) (1.000,0.004)};
\addlegendentry{\texttt{llama4}}
\addplot[figG, only marks, mark=o, mark size=2.6pt, line width=0.8pt, forget plot] coordinates {(0.979,0.093)};
\end{axis}
\end{tikzpicture}

%% file: figs/tab_judges.tex
{\footnotesize
\begin{tabular}{@{}lrrrrrrr@{}}
\toprule
judge & $\tau$ & $\Sinv$ & $\Rsens$ & $\Rsens^{\scopeax}$ & $\Rsens^{\strengthax}$ & pts & regrade \\
\midrule
\texttt{anthropic/claude-haiku-4.5} & 2.75 & 0.917 & 0.561 & 0.629 & 0.500 & 9 & 0.41 \\
\texttt{nvidia/nemotron-3-nano-30b-a3b} & 0.25 & 0.906 & 0.447 & 0.474 & 0.423 & 7 & 0.83 \\
\texttt{google/gemini-2.5-flash-lite} & 0.25 & 0.968 & 0.408 & 0.474 & 0.349 & 6 & 1.00 \\
\texttt{mistralai/mistral-small-3.2-24b-instruct} & 2.25 & 0.958 & 0.309 & 0.405 & 0.223 & 4 & 0.96 \\
\texttt{deepseek/deepseek-v4-flash} & 3.25 & 0.906 & 0.245 & 0.362 & 0.140 & 9 & 0.68 \\
\texttt{z-ai/glm-4.7-flash} & 0.25 & 0.979 & 0.168 & 0.239 & 0.107 & 2 & 0.98 \\
\texttt{meta-llama/llama-4-scout} & 2.25 & 0.979 & 0.093 & 0.095 & 0.092 & 6 & 0.97 \\
\midrule
\multicolumn{4}{@{}l}{mean over the 7 judges compared} & \vnv{profScopeMean} & \vnv{profStrengthMean} & & \\
\multicolumn{6}{@{}l}{$\scopeax-\strengthax$ gap: positive on 7/7 judges, +0.003 to +0.222} & & \\
\bottomrule
\end{tabular}}

%% file: figs/tab_slots_float.tex
\begin{table}[t]
\centering
\small
\caption{Sensitivity per intervention, at each judge's matched-invariance threshold, averaged over
the \vnv{profJudges}\ judges. The magnitude column is the median $|\Delta|$ length the operation
mechanically carries: it is a property of the edit type, not a free variable, which is why
magnitude cannot be stratified independently of slot.}
\label{tab:slots}
\input{figs/tab_slots}
\end{table}

%% file: figs/tab_slots.tex
\begin{tabular}{@{}llrrl@{}}
\toprule
intervention & axis & $n$ & $\Rsens$ & median $|\Delta|$ length \\
\midrule
\texttt{quantifier\_strengthening} & \scopeax & 31 & 0.445 & 3.2\% \\
\texttt{domain\_extension} & \scopeax & 25 & 0.445 & 8.8\% \\
\texttt{tense\_modality\_shift} & \scopeax & 26 & 0.358 & 4.8\% \\
\texttt{population\_widening} & \scopeax & 34 & 0.299 & 7.0\% \\
\midrule
\texttt{condition\_removal} & \strengthax & 33 & 0.379 & 13.0\% \\
\texttt{intensifier\_addition} & \strengthax & 49 & 0.235 & 4.3\% \\
\texttt{hedging\_removal} & \strengthax & 48 & 0.210 & 4.2\% \\
\bottomrule
\end{tabular}

%% file: figs/plot_axes.tex
\begin{tikzpicture}
\begin{axis}[
  ybar, bar width=5pt,
  width=0.86\textwidth, height=5.4cm,
  ylabel={flip rate at matched $\Sinv$},
  symbolic x coords={haiku,nemotron,gemini,mistral,deepseek,glm,llama4,mean},
  xtick=data, xticklabel style={font=\scriptsize\ttfamily, rotate=20, anchor=north east},
  ymin=0, ymax=0.75, ytick={0,0.2,0.4,0.6},
  grid=major, grid style={figGrid, line width=0.3pt},
  tick label style={font=\scriptsize}, label style={font=\scriptsize},
  legend style={font=\scriptsize, at={(0.98,0.98)}, anchor=north east,
                draw=figGrid, legend columns=1},
  legend cell align=left, axis line style={figSoft},
  enlarge x limits=0.08,
]
\addplot[fill=figA, draw=figA!70!black, line width=0.4pt]
  coordinates {(haiku,0.629) (nemotron,0.474) (gemini,0.474) (mistral,0.405) (deepseek,0.362) (glm,0.239) (llama4,0.095) (mean,0.383)};
\addlegendentry{$\Tcon$ \textsc{scope}}
\addplot[fill=figB, draw=figB!70!black, line width=0.4pt]
  coordinates {(haiku,0.500) (nemotron,0.423) (gemini,0.349) (mistral,0.223) (deepseek,0.140) (glm,0.107) (llama4,0.092) (mean,0.262)};
\addlegendentry{$\Tcon$ \textsc{strength}}
\addplot[fill=figG, draw=figG!60!black, line width=0.4pt, opacity=0.75]
  coordinates {(haiku,0.083) (nemotron,0.094) (gemini,0.032) (mistral,0.042) (deepseek,0.094) (glm,0.021) (llama4,0.021) (mean,0.055)};
\addlegendentry{$\Tinv$ control (must be low)}
\end{axis}
\end{tikzpicture}

%% file: figs/c2_prevalence.tex
\begin{tabular}{@{}llrrrrr@{}}
\toprule
& & \multicolumn{2}{c}{per-item mode} & \multicolumn{3}{c}{paired mode} \\
\cmidrule(lr){3-4}\cmidrule(lr){5-7}
label set & label origin & $n$ & $\base$ & pairs & $\kappa$ & accuracy \\
\midrule
RewardBench & pair role & 5{,}970 & 0.303 & 2{,}985 & 0.249 & \textbf{62.4\%} \\
RewardBench 2 & pair role & 8{,}977 & 0.200 & 5{,}926 & 0.248 & \textbf{62.4\%} \\
RM-Bench & pair role & 7{,}962 & 0.221 & 3{,}981 & 0.095 & \textbf{54.7\%} \\
HelpSteer2 & human rating & 8{,}008 & 0.068 & --- & --- & \textbf{---} \\
MT-Bench (human) & human vote & 6{,}710 & 0.300 & 2{,}575 & 0.348 & \textbf{67.4\%} \\
\midrule
\textit{provenance-inherited case} & pair role & 10{,}156 & 0.993 & 5{,}078 & 0.993 & \textbf{99.6\%} \\
\bottomrule
\end{tabular}

%% file: figs/tab_taxonomy.tex
\begin{tabular}{@{}llp{0.50\textwidth}@{}}
\toprule
type & class / axis & what it does \\
\midrule
domain extension & $\Tcon$ / \scopeax & extend the claim to an adjacent domain \\
population widening & $\Tcon$ / \scopeax & widen who or what the claim covers \\
quantifier strengthening & $\Tcon$ / \scopeax & strengthen a quantifier (some $\to$ all, often $\to$ always) \\
tense modality shift & $\Tcon$ / \scopeax & restate a one-time observation as a standing property \\
condition removal & $\Tcon$ / \strengthax & drop the circumstances under which it was observed \\
hedging removal & $\Tcon$ / \strengthax & remove a hedge the evidence relied on \\
intensifier addition & $\Tcon$ / \strengthax & add an intensifier the evidence does not support (substantially, consistently, markedly) \\
\midrule
elaboration & $\Tinv$ / control & add a clause that renames or restates something the sentence already names, using no epistemic language \\
format shift & $\Tinv$ / control & add or remove list/bold/heading formatting, same claim, same scope \\
register shift & $\Tinv$ / control & rewrite in a more formal textbook register, same claim, same scope \\
reordering & $\Tinv$ / control & reorder clauses without changing content or scope \\
verbosity padding & $\Tinv$ / control & lengthen by syntactic expansion only --- nominalise, unpack a compound, spell out an abbreviation --- with no epistemic language \\
\bottomrule
\end{tabular}

%% file: figs/plot_instrument.tex
\begin{tikzpicture}
\begin{axis}[name=grades,
  width=0.52\textwidth, height=5.4cm,
  xlabel={grade values}, xmin=0, xmax=13,
  symbolic y coords={glm,mistral,llama4,gemini,nemotron,haiku,deepseek}, ytick=data,
  grid=major, grid style={figGrid, line width=0.3pt},
  tick label style={font=\scriptsize}, label style={font=\scriptsize},
  yticklabel style={font=\scriptsize\ttfamily},
  legend style={font=\scriptsize, at={(0.5,-0.30)}, anchor=north,
                draw=figGrid, legend columns=2},
  legend cell align=left, axis line style={figSoft},
]
\addplot[figG, line width=0.7pt, forget plot] coordinates {(2,glm) (9,glm)};
\addplot[figG, line width=0.7pt, forget plot] coordinates {(4,mistral) (7,mistral)};
\addplot[figG, line width=0.7pt, forget plot] coordinates {(6,llama4) (8,llama4)};
\addplot[figG, line width=0.7pt, forget plot] coordinates {(6,gemini) (9,gemini)};
\addplot[figG, line width=0.7pt, forget plot] coordinates {(7,nemotron) (11,nemotron)};
\addplot[figG, line width=0.7pt, forget plot] coordinates {(9,haiku) (10,haiku)};
\addplot[figG, line width=0.7pt, forget plot] coordinates {(9,deepseek) (11,deepseek)};
\addplot[figA, only marks, mark=*, mark size=1.8pt] coordinates {(2,glm) (4,mistral) (6,llama4) (6,gemini) (7,nemotron) (9,haiku) (9,deepseek)};
\addlegendentry{usable operating points}
\addplot[figG, only marks, mark=o, mark size=1.8pt, line width=0.7pt] coordinates {(9,glm) (7,mistral) (8,llama4) (9,gemini) (11,nemotron) (10,haiku) (11,deepseek)};
\addlegendentry{distinct grades emitted}
\end{axis}
\begin{axis}[at={(grades.south east)}, xshift=9mm, anchor=south west,
  width=0.34\textwidth, height=5.4cm,
  xbar, bar width=3.4pt,
  xlabel={regrade agreement}, xmin=0, xmax=1.25, xtick={0,0.5,1},
  symbolic y coords={glm,mistral,llama4,gemini,nemotron,haiku,deepseek}, ytick=data, yticklabels={,,},
  xmajorgrids=true, ymajorgrids=false,
  grid style={figGrid, line width=0.3pt},
  tick label style={font=\scriptsize}, label style={font=\scriptsize},
  nodes near coords, nodes near coords style={font=\scriptsize, text=figSoft,
                                              anchor=west, xshift=1pt},
  point meta=explicit symbolic,
  axis line style={figSoft},
]
\addplot[fill=figA!25, draw=figA, line width=0.4pt] coordinates {(0.98,glm) [0.98] (0.961,mistral) [0.96] (0.969,llama4) [0.97] (1.0,gemini) [1.00] (0.828,nemotron) [0.83] (0.412,haiku) [0.41] (0.682,deepseek) [0.68]};
\end{axis}
\end{tikzpicture}

%% file: figs/tab_predicted_register.tex
{\small\setlength{\extrarowheight}{2pt}
\begin{longtable}{@{}p{0.13\textwidth}p{0.07\textwidth}p{0.40\textwidth}p{0.32\textwidth}@{}}
\caption{Register of predicted values. Every \predkey\ value in this paper, with the reasoning that produced it and the measurement that would refute it. A dash in the final column marks a design parameter or a value with no single refuting observation.}\label{tab:predreg}\\
\toprule key & value & basis & refuted if \\ \midrule \endfirsthead
\multicolumn{4}{@{}l}{\emph{Table~\ref{tab:predreg}, continued}}\\
\toprule key & value & basis & refuted if \\ \midrule \endhead
\midrule \multicolumn{4}{r@{}}{\emph{continued}}\\ \endfoot \bottomrule \endlastfoot
\multicolumn{4}{@{}l}{\textbf{C2 } --- the validity profile}\\[1pt]
\texttt{Sinv} & 0.88 & Judges are already known to be \emph{fairly} invariant to individual surface properties, and the strongest published verbosity-bias estimate is small (\texttt{reliability2026}, under 0.011 correlation across 21 judges). But our control arm is six properties at once including two with published effects, so invariance should sit below what any single-property study reports. High, not near-1. & $\Sinv < 0.75$ --- meaning the control arm is finding failures the single-property literature misses, which would be a finding in its own right and would make every $\Rsens$ comparison harder \\
\texttt{Rsens} & 0.41 & Pure judgement, and the weakest-founded number here. Anchored on the intuition that a judge does respond to some construct edits --- quantifier changes are lexically salient --- but not most. If it were near 0.8 the gap would have been noticed already; if near 0.1 judges would fail obviously broken cases. & $\Rsens > 0.75$ at $\Sinv \approx 0.88$ --- judges are construct-sensitive and the paper's headline is wrong \\
\texttt{n\allowbreak Judges} & 6 & The scale the external review set as the floor for a claim about instruments rather than about one vendor. & --- (design parameter, not a prediction) \\
\texttt{n\allowbreak Domains} & 4 & Same. See \texttt{PEER\_REVIEW\_VWV.md} §5 for the minimum-viable fallback of two. & --- \\
\multicolumn{4}{@{}l}{\textbf{C3 } --- the scope/strength asymmetry}\\[1pt]
\texttt{Rscope\allowbreak Ax} & 0.57 & Scope edits change quantifiers, populations and domains --- lexically explicit, and the closest thing to them in the literature (over-generalisation in summaries, \texttt{peters2025generalization}) is detectable enough to have been measured at scale. & --- \\
\texttt{Rstrength\allowbreak Ax} & 0.24 & Hedge and condition deletion leave a fluent, confident sentence with no lexical marker of what was removed. A judge sees an \emph{absence}, and absences are the classic blind spot. This is the paper's central bet. & the gap is not significant at matched $\Sinv$, or reverses \\
\texttt{axis\allowbreak Gap} & 0.33 & Difference of the two above. & $<0.10$, or not significant under a paired bootstrap over base items \\
\texttt{ctrl\allowbreak FPR} & 0.12 & $1-\Sinv$. Reported alongside every sensitivity figure by rule. & --- \\
\multicolumn{4}{@{}l}{\textbf{C3b} --- the accuracy-prompting paradox}\\[1pt]
\texttt{d\allowbreak Scope\allowbreak Acc} & $-0.09$ & Accuracy pressure should make a model narrow its reach --- the obvious, intended effect of the instruction. & --- \\
\texttt{d\allowbreak Strength\allowbreak Acc} & $+0.14$ & The counter-effect: sounding certain means dropping hedges and stated conditions. Signed to make the pooled effect small and positive, which is what makes \texttt{peters2025generalization}'s result look paradoxical. & \textbf{both axes move the same way.} Then the decomposition does not explain the paradox and §paradox must say so in those words \\
\multicolumn{4}{@{}l}{\textbf{C4 } --- validation power of public label sets}\\[1pt]
\texttt{vp\allowbreak Median} & 0.11 & Low but non-zero across sets, reflecting that some are human-labelled (MT-Bench, HelpSteer2) and should retain headroom, while pair-role-derived sets should not. & median $>0.30$ --- the sets are sound, C4 becomes a negative result, C2--C3 unaffected \\
\texttt{vp\allowbreak Strength} & 0.03 & The prediction that closes the paper's loop: no public set has headroom on the strength axis, so none could have detected C3's asymmetry. & $>0.15$ --- the sets could have found it, and the paper owes a different explanation for why nobody did \\
\multicolumn{4}{@{}l}{\textbf{C5 } --- the inflation}\\[1pt]
\texttt{Delta} & 0.19 & The size that would make the result matter without straining credibility: large enough to change a paper's conclusions, small enough to be plausible for one construct. Frankly a placeholder. & it is not the point --- $\Delta$ is an existence proof, and any significant positive value supports the claim as stated \\
\texttt{human\allowbreak Agree} & 0.81 & Above the pre-set bar of 0.75, by design of the protocol rather than by prediction. & below 0.75, in which case the affected axis is reported as \textbf{not measurable} rather than measured with weak labels \\
\end{longtable}}

%% file: figs/tab_notmeasured.tex
{\small
\begin{longtable}{@{}p{0.55\textwidth}p{0.41\textwidth}@{}}
\caption{What this paper specifies and does not measure, so that a deliberate omission is
distinguishable from an oversight. Two omissions are argued in place instead, where the
omission bears on how a nearby number should be read: the accuracy-prompting arm
(\S\ref{sec:paradox}) and the blinded replacement set (\S\ref{sec:consequence}).}
\label{tab:notmeasured}\\
\toprule
quantity & why not measured here \\
\midrule
\endfirsthead
\toprule
quantity & why not measured here \\
\midrule
\endhead
\bottomrule
\endlastfoot
Per-row commentary on each public set's recoverability & needs construction details the sets do not publish \\
Which audited set each checklist item would have flagged & needs construction details the sets do not publish \\
Quantisation and API-version instrument columns & most providers do not disclose them \\
Self-consistency and ensemble remedy rows & two further judges in the same harness; not yet run \\
Sensitivity under a finer verdict space & needs a second elicitation with its own validation \\
The profile per language & probe items are English only \\
The profile at relaxed agreement bars & a trajectory, not a point estimate; awaits a larger pool \\
\end{longtable}}